\documentclass[journal]{IEEEtran}

\usepackage{iftex}
\ifLuaTeX 
\usepackage[no-math]{fontspec}
\usepackage{graphicx}

\renewcommand{\IEEEPARstart}[2]{{\fontsize{13pt}{14pt}\selectfont\textbf{#1}}#2}

\else 
\usepackage[dvipdfmx]{graphicx}
\fi

\usepackage{url}

\usepackage{subcaption}
\usepackage{citesort}
\usepackage{stmaryrd, amsmath}
\usepackage{ascmac}
\usepackage{cite}
\usepackage{amsbsy}
\usepackage{amssymb}
\usepackage{latexsym}
\usepackage{xcolor}
\usepackage{empheq}

\usepackage{booktabs}
\usepackage{makecell}

\usepackage{algorithm}
\usepackage{algorithmicx, algpseudocode}

\usepackage{bm}
\usepackage{breqn}
\graphicspath{{figs/}{/Users/guest/research/code/2025/RL_sparse/chainwalk/figs/}}

\usepackage{mathtools}
\usepackage{enumitem}
\usepackage{amsthm}
\usepackage{mathrsfs}
\usepackage{cleveref}
\crefname{assumption}{Assumption}{Assumptions}
\crefname{problem}{Problem}{Problems}
\crefname{theorem}{Theorem}{Theorems}
\crefname{section}{Section}{Sections}
\crefname{subsection}{Section}{Sections}
\crefname{subsubsection}{Section}{Sections}
\crefname{remark}{Remark}{Remarks}
\crefname{fact}{Fact}{Facts}
\crefname{proposition}{Proposition}{Propositions}
\crefname{corollary}{Corollary}{Corollaries}
\crefname{lemma}{Lemma}{Lemmas}
\crefname{definition}{Definition}{Definitions}
\crefname{appendix}{Appendix}{Appendices}

\newtheorem{theorem}{Theorem}

\newtheorem{proposition}[theorem]{Proposition}

\newtheorem{definition}[theorem]{Definition}
\newtheorem{lemma}[theorem]{Lemma}

\newtheorem{assumption}[theorem]{Assumption}
\newtheorem{fact}[theorem]{Fact}
\newtheorem{remark}[theorem]{Remark}
\newtheorem{problem}[theorem]{Problem}

\newcommand{\dom}{\mathop{\rm dom}\nolimits}

\newcommand{\sProx}{\mathop{\rm s\mbox{-}Prox}\nolimits}

\newcommand{\soft}{\mathop{\rm soft}\nolimits}
\newcommand{\Soft}{\mathop{\rm Soft}\nolimits}

\newcommand{\sign}{\mathop{\rm sign}\nolimits}

\newcommand{\zer}{\mathop{\rm zer}\nolimits}
\newcommand{\Id}{\mathop{\rm Id}\nolimits}
\newcommand{\range}{\mathop{\rm range}\nolimits}

\newcommand{\gra}{\mathop{\rm gra}\nolimits}
\newcommand{\lev}{\mathop{\rm lev}\nolimits}

\newcommand{\diag}{\mathop{\rm diag}\nolimits}
\newcommand{\rank}{\mathop{\rm rank}\nolimits}
\newcommand{\Null}{\mathop{\rm null}\nolimits}
\newcommand{\Span}{\mathop{\rm span}\nolimits}
\newcommand{\clip}{\mathop{\rm clip}\nolimits}

\DeclareMathOperator*{\argmin}{arg\,min}

\newcommand{\zeromat}{\bm{0}}

\usepackage{eso-pic}
\begin{document}

\title{Non-Convex Sparse Reinforcement Learning via Non-Monotone Inclusions}

\author{Kyohei Suzuki,~\IEEEmembership{Member,~IEEE}, and Konstantinos Slavakis,~\IEEEmembership{Senior
    Member,~IEEE}
\thanks{K.~Suzuki and K.~Slavakis are with the Department of Information and Communications Engineering,
  Institute of Science Tokyo, Yokohama, 226-8501, Japan. E-mails: suzuki.k.439f@m.isct.ac.jp,
  slavakis@ict.eng.isct.ac.jp.}
\thanks{This work was supported by the Grants-in-Aid for Scientific Research (KAKENHI) under Grant Number
  25K24422. This study was carried out using the TSUBAME4.0 supercomputer at Institute of Science
  Tokyo. An earlier version of this paper was presented at a conference [DOI: 10.1109/ICASSP55912.2026.11463148].}}

\maketitle

\begin{abstract}
  This work delivers two key contributions: one to efficient feature selection in reinforcement learning
  (RL), the other to the theory of non-monotone inclusions. On the RL side, the estimation bias inherent
  in conventional regularization schemes is addressed by augmenting classical least-squares
  temporal-difference (LSTD) policy evaluation with the sparsity-inducing, non-convex projected minimax
  concave (PMC) penalty. Because the PMC penalty is weakly convex, the resulting fixed-point problem is
  no longer monotone; instead, it falls under a broader class of non-monotone inclusions involving the
  sum of a monotone Lipschitz operator and a hypomonotone operator. On the theory side, novel convergence
  conditions are developed for the forward-reflected-backward splitting (FRBS) method applied to this
  broader class of non-monotone inclusion problems. Under mild conditions, Lyapunov stability and the
  existence of a limit point of the sequence of FRBS iterates are established; alternatively, under the
  weak Minty variational inequality assumption, exact convergence is guaranteed. Numerical tests on
  benchmark datasets show that the proposed FRBS iterates, applied to the non-convexly regularized LSTD
  problem, substantially outperform state-of-the-art feature-selection methods, especially when many
  noisy features are present.
\end{abstract}



\section{Introduction}\label{sec:intro}

\IEEEPARstart{R}{einforcement} learning (RL) plays a central role in contemporary machine learning,
signal processing, and control theory~\cite{sutton1998reinforcement, bertsekas2019reinforcement,
  arulkumaran2017deep}. Its central goal is for an agent to learn, through interaction with an
environment typically modeled as a Markov decision process (MDP), an optimal policy that minimizes a
long-term loss captured by the Q-function. Yet in many real-world settings, such as robotics, educational
agents, and healthcare, repeated online interaction with the environment is costly, and the collected
data are often corrupted by erroneous measurements or outliers that can severely degrade the learned
policy~\cite{levine2020offline}. Robustness to such data imperfections is therefore essential, which
motivates batch (offline) RL, where the agent learns from a fixed, pre-collected dataset of transitions
rather than through additional online interaction.

In most high-dimensional, real-world problems, explicitly representing the Q-function for all possible
states and actions is impractical due to the ``curse of dimensionality.'' A common remedy is to
approximate the Q-function using a functional (non-)parametric representation. This, however, introduces
a fundamental trade-off between approximation accuracy and computational complexity: reducing the
approximation error generally requires a large number of features in the model, which in turn increases
computational demands. This manuscript focuses on ``linear function'' approximation---``linear'' in some
feature space---which, despite the success of deep RL, remains actively studied for its interpretability
and its amenability to rigorous convergence analysis.

\subsection{Feature selection in RL}

Feature selection, achieved via a sparse representation over a large basis of functions, is an effective
way to alleviate the trade-off between approximation accuracy and computational complexity, mitigate
overfitting, improve sample efficiency, and promote robustness against irrelevant features and
outliers. Most existing work on feature selection in RL has focused on $\ell_1$-norm regularization to
induce sparsity, including batch algorithms such as least-angle regression
(LARS)-TD~\cite{kolter2009regularization}, $\ell_1$-projected Bellman
residual~(PBR)~\cite{geist2011ell1}, and basis-pursuit denoising (BPDN)~\cite{qin2014sparse}
(see~\cite{liu2015feature} and references therein), as well as online
algorithms~\cite{mahadevan2012sparse, song2018sparse, song2021online, painter2012l1}. In particular,
LARS-TD applies the LARS algorithm to solve a fixed-point equation for the Bellman mapping associated
with the $\ell_1$-regularized least-squares temporal-difference (LSTD) formulation, even though this task
is known not to admit a convex optimization formulation~\cite{kolter2009regularization}. In contrast,
methods~\cite{geist2011ell1, qin2014sparse, song2021online} consider the projected Bellman residual
minimization problem regularized by the $\ell_1$ norm, which is a convex optimization task. However, this
approach generally requires projecting onto the span of the feature vectors, yielding higher time and
memory complexities than LARS-TD~\cite{geist2011ell1}.

While many feature selection methods are originally formulated for value-function approximation, they can
also be adapted to Q-function settings for learning an optimal policy, as in omnibus
approximate-policy-iteration (API) schemes; see, for example, the least-squares policy iteration
(LSPI)~\cite{lagoudakis2003least}. Usually in API, the approximation error relative to the optimal
Q-function is bounded by $2\gamma / (1-\gamma)^2$ times the uniform upper bound on the per-iteration
policy-evaluation error, where $\gamma$ is the discount factor~\cite{lagoudakis2003least,
  bertsekas2019reinforcement}.

\subsection{Sparse regression}

Sparse regression is a fundamental technique in high-dimensional data analysis. The least absolute
shrinkage and selection operator (LASSO)~\cite{tibshirani1996regression, hastie2015statistical,
  wainwright2019high}, which employs the $\ell_1$-norm penalty, is widely used for sparse
regression. Despite its tractability, it suffers from estimation bias, underestimating large-amplitude
components~\cite{fan2001variable, osher2016sparse}. To mitigate this limitation, non-convex penalties
have been proposed, including the minimax concave (MC) penalty~\cite{zhang2010nearly}, smoothly clipped
absolute deviation (SCAD)~\cite{fan2001variable}, and the $\ell_q$-quasi-norm for $q \in (0,
1)$~\cite{chartrand2007exact}, with particular focus on those that preserve the overall convexity of the
cost function when combined with a least-squares model~\cite{selesnick2017sparse, bayram2016convergence,
  abe2020linearly, yukawa2023linearly}. In particular, the projective minimax concave (PMC)
penalty~\cite{yukawa2023linearly} ensures overall convexity even in underdetermined settings by
annihilating the debiasing effects on the orthogonal complement of the input subspace, while allowing
efficient implementation without introducing auxiliary variables (see \cref{subsec:Projective Minimax
  Concave Penalty}).

Since the approximation accuracy of API degrades when the per-iteration policy-evaluation error is large,
sparsity-inducing non-convex penalties are anticipated to substantially improve the accuracy of feature
selection for RL. Despite this intuition, non-convex sparsity-inducing penalties remain unexplored in RL,
even for batch algorithms. This raises a natural question: \textit{can such penalties be effectively
  employed for feature selection in RL while ensuring convergence guarantees?}

\subsection{Contributions}

This paper provides an affirmative answer to the previous research question. The contributions of this
work are summarized below.

\begin{enumerate}

\item An efficient batch method for feature selection in RL is introduced by formulating a fixed-point
  problem that incorporates the non-convex PMC penalty into the classical LSTD
  formulation~\cite{bradtke1996linear, lagoudakis2003least}. By exploiting the weak convexity of PMC, the
  problem is reformulated as finding a zero of the sum of a monotone Lipschitz operator and a
  hypomonotone operator (\cref{PROP:REFORMULATION} and \eqref{eq:reforemulated_problem}).

  Furthermore, a closed-form expression of the resolvent of the hypomonotone part is derived
  (\cref{PROP:RESOLVENT_OF_A}), making the resulting problem tractable via the forward-reflected-backward
  splitting (FRBS) method~\cite{malitsky2020forward}. An error bound for the policies generated by an API
  based method is also provided (\cref{prop:convergence_API}), following the analysis
  of~\cite{lagoudakis2003least}.

\item Since the existing convergence analysis in~\cite{malitsky2020forward} is applicable only to a
  monotone inclusion problem, extended convergence conditions are provided for a general non-monotone
  inclusion problem, which encompasses the proposed non-convexly regularized LSTD setting as a special
  case. Specifically, the following two analyses are performed.
  \begin{enumerate}
  \item Without additional structural assumptions, quasi-Fej\'{e}r monotonicity of the iterates with
    respect to the solution set is established (\cref{THEOREM:CONVERGENCE_GUARANTEE}). This yields
    Lyapunov stability and the existence of a limit of the iterates under mild conditions
    (\cref{proposition:Lyapunov,proposition:convergent_sequence}).
  \item Under the weak Minty variational inequality (MVI) assumption~\cite{diakonikolas2021efficient},
    exact convergence of the sequence to a solution is established.
  \end{enumerate}

\item While, in the original definition of the PMC penalty in \cite{yukawa2023linearly}, the subspace on
  which the debiasing effect is confined is fixed to the orthogonal complement of the input subspace,
  this paper generalizes it to an arbitrary subspace controlled by a tunable hyperparameter $q$
  (\cref{PROP:REFORMULATION}). Choosing $q$ appropriately prevents the performance degradation of the PMC
  penalty, especially when the smallest positive eigenvalue of the input Gram matrix is very small (see
  \cref{remark:PMC} for details).

\item Numerical tests on three benchmark RL tasks demonstrate that the proposed approach substantially
  outperforms state-of-the-art batch feature-selection methods, even in scenarios with numerous noisy
  features. The influence of the hyperparameter $q$ is also analyzed, showing that a moderate range of
  $q$ enhances the debiasing effect of the PMC penalty.

\end{enumerate}

\subsection{Prior art on non-monotone inclusion problems}

The theoretical study of non-monotone inclusion problems has attracted growing attention in recent years,
motivated by non-convex/non-concave min-max problems. A common approach is to relax the monotonicity
condition to the weak MVI condition~\cite{diakonikolas2021efficient, alacaoglu2024extending,
  alacaoglu2023beyond, bohm2022solving, dang2015convergence}, including extensions of the
Douglas-Rachford method~\cite{alcantara2025douglas, evens2025convergence}. While the Douglas-Rachford
method requires computing the resolvent of each constituent operator, this work adopts the FRBS
method~\cite{malitsky2020forward}, which requires only the resolvent of the hypomonotone part. While
convergence of FRBS is guaranteed only for monotone inclusions~\cite{malitsky2020forward}, several
extensions to treat a non-monotone operator have been proposed recently~\cite{tran2025variance,
  van2025forward}. However, these works consider operator structures and settings different from those of
this paper (see \cref{remark:relations_monotone_FRBS}).

\subsection{Structure of the manuscript}

The manuscript is organized as follows. \cref{sec:preliminaries} introduces the notation and definitions
used in this paper, along with a brief review of the background on RL, preliminaries on convex analysis,
and the PMC penalty. \cref{sec:main_results} presents the proposed formulation based on the PMC
penalty. \cref{sec:convergence_nonmonotone_inclusion} establishes the extended convergence conditions of
FRBS for a general class of non-monotone inclusion problems. \cref{sec:numerical_tests} presents the
simulation results, and \cref{sec:conclusions} provides concluding remarks. This work extends
substantially its short preliminary version \cite{suzuki2026nonconvex} by (i) establishing the
convergence analysis under the weak MVI assumption, (ii) providing proofs for the theoretical results,
and (iii) demonstrating the effectiveness of the proposed method via extensive numerical tests.

\section{Preliminaries}\label{sec:preliminaries}

Throughout the paper, let $\mathbb{R}$, $\mathbb{R}_{+}$, $\mathbb{R}_{++}$, and $\mathbb{N}$ denote the
sets of real numbers, nonnegative real numbers, strictly positive real numbers, and nonnegative integers,
respectively. Let also $\mathbb{N}_* \coloneqq \mathbb{N} \setminus \{0\}$. For $n_1, n_2\in
\mathbb{N}_*$, with $n_1 \leq n_2$, define $\overline{n_1, n_2} \coloneqq \{ n_1, n_1+1, \ldots,
n_2\}$. Vector/matrix transposition is denoted by $(\cdot)^{\top}$.  For any $n \in \mathbb{N}^*$ and $p
\in [1, +\infty)$, the $\ell_p$ norm of a vector $\bm{x} \coloneqq [x_1, x_2, \ldots, x_n]^{\top} \in
  \mathbb{R}^n$ is defined by $\| \bm{x} \|_p \coloneqq \left(\sum_{i=1}^n \left| x_i \right|^p
  \right)^{1 / p}$. For any $n \in \mathbb{N}^*$ and $x_1, x_2, \ldots, x_n \in \mathbb{R}$, let
  $\diag(x_1, x_2, \ldots, x_n)$ denote the diagonal matrix whose $i$th diagonal entry is $x_i$ for every
  $i = 1,2, \ldots, n$.  Let $\zeromat_{m \times n}$ and $\bm{0}_n$ denote the $m \times n$ zero matrix
  and $n$-dimensional zero vector, respectively.

\subsection{Background on RL}\label{subsec:Background on reinforcement learning}

In RL problems, the interaction between an environment and an agent is modeled as an MDP defined as the
tuple $(\mathfrak{S}, \mathfrak{A}, \mathcal{P}, \gamma, g)$, where $\mathfrak{S}$ is a finite state
space for simplicity, $\mathfrak{A}$ is a finite action space, $\mathcal{P}\colon \mathfrak{S} \times
\mathfrak{S} \times \mathfrak{A} \rightarrow [0, 1]$ is a state-transition probability such that
$\mathcal{P}(s^{\prime} \mid s, a)$ is the probability of transitioning to state $s^{\prime}$ when taking
action $a$ in state $s$, $\gamma \in (0, 1)$ is a discount factor, and $g \in \mathcal{B}$ is the
one-step loss. Here, $\mathcal{B}$ denotes the Banach space of bounded real-valued functions over
$\mathfrak{Z} \coloneqq \mathfrak{S} \times \mathfrak{A}$, equipped with the $L_{\infty}$ norm defined by
$\| \cdot \|_{\infty} \colon \mathcal{B} \to \mathbb{R} \colon Q \mapsto \sup_{(s, a) \in \mathfrak{Z}}
|Q(s, a)|$.  A deterministic stationary policy $\pi \colon \mathfrak{S} \rightarrow \mathfrak{A}$
specifies the action $\pi(s)$ that the agent takes in state $s$.

Let $(S_t, A_t)$ denote the state-action random variables at time $t$. Then, the Q-function (Q-factor)
$Q^{\pi}(s, a)$ of a policy $\pi$ at state-action pair $(s, a)$ is defined as the expected cost starting
from state $s$ and action $a$ and subsequently using policy $\pi$, \textit{i.e.},
\begin{align}
  &\mathcal{B} \ni Q^{\pi}  \colon (s, a) \mapsto \nonumber \\
  &g(s, a ) + \lim_{N \to +\infty} \mathbb{E} \left\{ \sum_{t = 1}^{N - 1} \gamma^{t} g (S_t, \pi(S_t)) \mid S_0 = s, A_0 = a \right\},
  \label{eq:Q_function}
\end{align}
where the conditional probability of observing the trajectory $(s_t, \pi(s_t))_{t = 1}^{+\infty}$ given
$S_0 = s$ and $A_0 = a$ is given by $\mathcal{P}(s_1 \mid s, a) \prod_{t = 2}^{+\infty} \mathcal{P}(s_t
\mid s_{t - 1}, \pi(s_{t - 1}))$.  The optimal Q-function is defined as
\begin{equation}
  \mathcal{B} \ni Q^{*} \colon (s, a) \mapsto \inf_{\pi \in \Pi} Q^{\pi}(s, a),
  \label{eq:optimal_Q}
\end{equation}
where $\Pi$ is the set of all policies. Given a mapping $T \colon \mathcal{B} \to \mathcal{B}$, $Q \in
\mathcal{B}$ is called a fixed point of $T$ if $Q = TQ$.  It is well known that $Q^{\pi}$ and $Q^*$ are
the unique fixed points of the Bellman mapping $T_{\pi}$ and Bellman optimality mapping $T_{*}$,
respectively, i.e.,
\begin{subequations}
  \begin{align}
    &Q^{\pi} = T_{\pi} Q^{\pi}, \label{eq:Bellman_equation} \\
    &Q^{*} = T_{*} Q^{*}, \label{eq:Bellman_optimality_equation}
  \end{align}
\end{subequations}
where $T_{\pi} \colon \mathcal{B} \to \mathcal{B}$ and $T_* \colon \mathcal{B} \to \mathcal{B}$ are
defined as follows, respectively: $\forall Q \in \mathcal{B}$,
\begin{subequations}
  \begin{align}
    &\mathcal{B} \ni T_{\pi} Q \colon (s,a) \mapsto \nonumber \\
      &\quad g(s, a) + \gamma \mathbb{E} \{ Q(S_{t+1}, \pi(S_{t+1})) \mid S_t = s, A_t = a \}, \\
    &\mathcal{B} \ni T_{*} Q \colon (s,a) \mapsto \nonumber \\
    &\quad g(s, a) + \gamma \mathbb{E} \{ \min_{a' \in \mathfrak{A}} Q(S_{t+1}, a') \mid S_t = s, A_t = a \}.
    \label{eq:Bellman_optimality_mapping}
  \end{align}
\end{subequations}
The equations \eqref{eq:Bellman_equation} and \eqref{eq:Bellman_optimality_equation} are called the Bellman equation and the Bellman optimality equation, respectively.
Note that both $T_{\pi}$ and $T_*$ are $\gamma$-contraction mappings, \textit{i.e.},
\begin{align}
  \|T_{\pi} Q_1 - T_{\pi} Q_2 \|_{\infty} \le \gamma \|Q_1 - Q_2 \|_{\infty}, ~ \forall Q_1, Q_2 \in \mathcal{B}, \\
  \|T_* Q_1 - T_* Q_2 \|_{\infty} \le \gamma \|Q_1 - Q_2 \|_{\infty}, ~ \forall Q_1, Q_2 \in \mathcal{B},
\end{align}
and hence each has a unique fixed point~\cite{bertsekas2012dynamic}.

When the number of states is exceedingly large or the state space is continuous, evaluating $Q^{\pi}$ at
every state-action pair becomes computationally prohibitive. A common remedy is to approximate $Q^{\pi}$
by a ``linear function,'' $Q^{\pi}(s, a) \approx \hat{Q}_{\bm{w}}^{\pi} (s, a) \coloneqq \bm{w}^{\top}
\bm{\phi} (s, a) $, which is linear in $\bm{w} \in \mathbb{R}^n$---the vector collecting all learnable
parameters of the approximation---for some $n \in \mathbb{N}_*$, where $\bm{\phi}(s, a) \in \mathbb{R}^n$
denotes the feature vector at $(s, a) \in \mathfrak{Z}$. Although $T_{\pi}
\hat{Q}_{\bm{w}}^{\pi}$ may not lie in the span of the feature vectors, the temporal-difference (TD)
family of algorithms solves for a $\bm{w}$ such that (s.t.) its corresponding $\hat{Q}_{\bm{w}}^{\pi}$
approximately satisfies the fixed-point problem \eqref{eq:Bellman_equation}, \textit{i.e.},
$\hat{Q}_{\bm{u}}^{\pi} \approx T_{\pi} \hat{Q}_{\bm{w}}^{\pi} $:
\begin{subequations}
  \begin{align}
    \mathrm{find}~ \bm{w} \in \mathbb{R}^n ~ \text{s.t.}~ \bm{w} \in T_{\mathrm{LS}}(\bm{w}),
    \label{eq:Bellman_residual}
  \end{align}
  where
  \begin{align}
    &T_{\mathrm{LS}} \colon \mathbb{R}^n \to 2^{\mathbb{R}^n} \nonumber\\
    &\quad  \bm{w} \mapsto
    \argmin_{\bm{u} \in \mathbb{R}^n}
    \frac{1}{2} \sum_{(s, a) \in \mathfrak{Z}} \!\left[\hat{Q}_{\bm{u}}^{\pi} (s, a) - T_{\pi} \hat{Q}_{\bm{w}}^{\pi} (s, a) \right]^2. \label{eq:T_LSTD}
  \end{align}
\end{subequations}
Since $\mathcal{P}$ is in general unavailable,
LSTD~\cite{bradtke1996linear,lagoudakis2003least} approximates \eqref{eq:T_LSTD} via $m$ samples $(s_i, a_i, g(s_i, a_i), s_i')_{i= 1}^m$:
\begin{subequations}
  \begin{align}
    \mathrm{find}~\bm{w} \in \mathbb{R}^n\ \text{s.t.}\ \bm{w} \in T_{\mathrm{LSTD}}(\bm{w}), \label{eq:LSTD_formulation}
  \end{align}
  where
  \begin{align}
    &T_{\mathrm{LSTD}}: \mathbb{R}^n \to 2^{\mathbb{R}^n} \nonumber\\
    &\quad \bm{w} \mapsto \argmin_{\bm{u} \in \mathbb{R}^n} \frac{1}{2}
    \| \bm{\Phi} \bm{u} - (\bm{g} + \gamma \bm{\Phi}' \bm{w} )
    \|_2^2.
  \end{align}
\end{subequations}
Here, the $m\times n$ matrices $\bm{\Phi} \coloneqq [\bm{\phi}(s_1, a_1), \ldots, \bm{\phi}(s_m,
  a_m)]^{\top}$, $\bm{\Phi}' \coloneqq [ \bm{\phi}(s'_1, \pi(s'_1)), \ldots, \bm{\phi}(s'_m, \pi(s'_m))
]^{\top}$, and the $m\times 1$ vector $\bm{g} \coloneqq [ g(s_1, a_1), \ldots, g(s_m, a_m) ]^{\top}$. A
solution to \eqref{eq:LSTD_formulation} can be obtained analytically by $\bm{w}_* = \bm{\Omega}^{-1}
\bm{b}$, where
\begin{align}
  \bm{\Omega} \coloneqq \bm{\Phi}^{\top} (\bm{\Phi} - \gamma
  \bm{\Phi}'), \quad \bm{b} \coloneqq \bm{\Phi}^{\top}
  \bm{g}. \label{def:A_tilde_b_tilde}
\end{align}
If $\bm{\Omega}$ is not invertible, one can instead employ the Moore-Penrose pseudoinverse of
$\bm{\Omega}$ or ridge regression~\cite{lagoudakis2003least, yu2009convergence}.

\subsection{Background on convex analysis}\label{subsec:selected_elements}

A function $f\colon \mathbb{R}^n \to (-\infty, +\infty] \coloneqq \mathbb{R} \cup \{+ \infty\}$ is proper
if $\dom f \coloneqq \{\bm{x} \in \mathbb{R}^n \mid f(\bm{x}) < +\infty\} \neq \emptyset$.  A function
$f\colon \mathbb{R}^n \to (-\infty, +\infty] \coloneqq \mathbb{R} \cup \{+ \infty\}$ is
  lower-semicontinuous on $\mathbb{R}^n$ if the level set $\lev_{\le a} f \coloneqq \{\bm{x} \in
  \mathbb{R}^n \mid f(\bm{x}) \le a\}$ is closed for any $a \in \mathbb{R}$. A function $f\colon
  \mathbb{R}^n \to (-\infty, +\infty]$ is convex if $f(a \bm{x} + (1-a) \bm{\xi}) \leq a f(\bm{x})+(1-a)
    f(\bm{\xi})$ for any $\bm{x}, \bm{\xi} \in \mathbb{R}^n$ and any $a \in(0,1)$.  For any $\rho > 0$, a
    function $f \colon \mathbb{R}^n \rightarrow (-\infty, +\infty]$ is $\rho$-weakly convex if $f + \rho
      \|\cdot\|_2^2 / 2$ is convex.  Let $\Gamma_0(\mathbb{R}^n)$ denote the set of proper
      lower-semicontinuous convex functions from $\mathbb{R}^n$ to $(-\infty, +\infty]$.

 Given a proper function $f \colon \mathbb{R}^n \to (-\infty, +\infty]$, the single-valued proximity
   operator of $f$ of index $\tau > 0$ is defined as~\cite{yukawa2025monotone}
    \begin{equation}
      \sProx_{\tau f}\colon
      \mathbb{R}^n \to \mathbb{R}^n \colon \bm{x} \mapsto \argmin_{\bm{\xi} \in \mathbb{R}^n} \left(
      f(\bm{\xi}) + \frac{1}{2 \tau}\|\bm{x} - \bm{\xi}\|_2^2 \right),
    \end{equation}
    whenever $f + \|\bm{x} - \cdot\|_2^2 / (2\tau)$ has a unique minimizer for every fixed $\bm{x} \in
    \mathbb{R}^n$.\footnote{If $f$ is non-convex, the proximity operator is often defined as a set-valued
    operator. This paper focuses on the case when the proximity operator is a single-valued operator, and
    the notation $\sProx$ is used to emphasize that.}  Given a function $f \in \Gamma_0(\mathbb{R}^n)$,
    the Moreau envelope of $f$ of index $\tau > 0$ is defined as~\cite{moreau1962decomposition,
      moreau1962fonctions, bauschke2017convex}
      \begin{equation}
        {}^{\tau} \! f \colon
        \mathbb{R}^n \to \mathbb{R} \colon \bm{x} \mapsto \min_{\bm{\xi} \in \mathbb{R}^n} \left(
        f(\bm{\xi}) + \frac{1}{2 \tau}\|\bm{x} - \bm{\xi}\|_2^2 \right).
      \end{equation}
    Note that it holds for any $f \in \Gamma_0(\mathbb{R}^n)$ that (i) its Moreau envelope ${}^{\tau} \!
    f$ is convex \cite[Proposition 12.15]{bauschke2017convex} and (ii) ${}^\tau f (\bm{x}) \downarrow
    \inf_{\bm{\xi} \in \mathbb{R}^n} f(\bm{\xi})$ as $\tau \uparrow + \infty$ for any $\bm{x} \in
    \mathbb{R}^n$ \cite[Proposition 12.33]{bauschke2017convex}. Furthermore, the following fact is used.
      \begin{fact}[{\cite[Proposition 12.30]{bauschke2017convex}}]
        \label{fact:prox_nabla}
        Let $f \in \Gamma_0(\mathbb{R}^n)$ and $\tau > 0$.
        Then, $\nabla({}^\tau f) = \tau^{-1}(\Id - \sProx_{\tau f})$ is $\tau^{-1}$-Lipschitz continuous.
    \end{fact}
      For any nonempty closed convex set $C \subset \mathbb{R}^n$, let $\iota_C \in
      \Gamma_0(\mathbb{R}^n)$ be the indicator function of $C$, defined as $\iota_C \colon \bm{x} \mapsto
      0$ if $\bm{x} \in C$, and $+\infty$ otherwise. The metric projection operator onto $C$, denoted by
      $P_C$, is defined as the proximity operator of $\iota_C$, i.e., $P_C \coloneqq \sProx_{\iota_C}$.

A set-valued mapping $T\colon \mathbb{R}^n \to 2^{\mathbb{R}^n}$ is called monotone if $\langle \bm{x} -
\bm{y}, \bm{u} - \bm{v} \rangle_2 \ge 0$, $\forall (\bm{x}, \bm{u}), (\bm{y}, \bm{v}) \in \gra T$, where
$\gra T \coloneqq \{(\bm{x}, \bm{u}) \in \mathbb{R}^n \times \mathbb{R}^n \mid \bm{u} \in T(\bm{x})\}$
denotes the graph of $T$.  A monotone operator $T\colon \mathbb{R}^n \to 2^{\mathbb{R}^n}$ is maximally
monotone if no other monotone operator has its graph containing $\gra T$.  If $T - \rho \Id$ is
(maximally) monotone for some $\rho \in \mathbb{R}$, $T$ is (maximally) $\rho$-monotone
\cite{bauschke2021generalized}, also known as $\rho$-strongly monotone when $\rho > 0$ and as
$|\rho|$-hypomonotone when $\rho < 0$.

\subsection{The projective minimax concave penalty}\label{subsec:Projective Minimax Concave Penalty}

The MC penalty \cite{zhang2010nearly} with index $\tau > 0$ is defined as 
\begin{align}
  \Psi_\tau^{\mathrm{MC}}\colon \mathbb{R}^n
\to [0,+\infty) \colon \bm{x} \mapsto \sum_{i=1}^n \psi_\tau^{\mathrm{MC}} \left(x_i\right),
\end{align}
where 
\begin{align}
  \psi_\tau^{\mathrm{MC}} \colon \mathbb{R} \to [0,+\infty) \colon x \mapsto
    \begin{cases}
        |x|- x^2 / (2 \tau), & \text { if }|x| \leq \tau, \\
        \tau / 2, & \text { if }|x|>\tau. \label{def:MC}
    \end{cases}
\end{align}
  The MC penalty can be expressed as a difference-of-convex (DC) function formulation known as the Moreau-enhanced model of the $\ell_1$ norm \cite{abe2020linearly},
  \textit{i.e.}, the difference between the $\ell_1$ norm and its Moreau envelope \cite{selesnick2017sparse}:
  \begin{equation}
    \Psi_{\tau}^{\mathrm{MC}} = \|\cdot\|_1 - \hspace{.3em}^{\tau} \|\cdot\|_1. \label{eq:MC_L1_Moreau}
  \end{equation}
  A key property of the MC penalty is its weak convexity \cite{selesnick2017sparse}.  Owing to this
  property, the overall convexity of the least-squares loss regularized by the MC penalty can be
  preserved by choosing the regularization parameter appropriately.  However, in the underdetermined
  case, the least-squares loss is no longer strongly convex since its Hessian is rank-deficient, and
  hence, the overall convexity cannot be guaranteed in general.

  The PMC penalty~\cite{yukawa2023linearly} defined below overcomes this difficulty by introducing a
  metric projection onto a given linear subspace $\mathcal{M} \subset \mathbb{R}^n$:
    \begin{equation}
    \Psi_{\tau, \mathcal{M}}^{\mathrm{PMC}} \coloneqq \|\cdot\|_1 - \hspace{.3em}^{\tau} \|\cdot\|_1 (P_{\mathcal{M}}
    \cdot ). \label{def:PMC}
  \end{equation}
  While the PMC penalty was originally proposed for the least-squares loss in
  underdetermined linear systems with $\mathcal{M}$ defined as the orthogonal complement of the null
  space of the input matrix, this paper considers $\mathcal{M}$ as a general subspace in $\mathbb{R}^n$
  (see also \cref{remark:PMC} in \cref{sec:main_results}).

  In fact, the PMC penalty coincides with the MC penalty on $\mathcal{M}$ and with the $\ell_1$ norm on
  $\mathcal{M}^{\perp}$~\cite[Proposition 1]{yukawa2023linearly}.  The motivation for introducing
  $P_{\mathcal{M}}$ is twofold: (i) to restrict the weak convexity to $\mathcal{M}$, where the loss
  function is strongly convex, and (ii) to keep the penalty convex on $\mathcal{M}^{\perp}$, where the
  loss function is flat (\textit{i.e.}, it possesses zero curvature).  As a result, the debiasing effect
  can be introduced on $\mathcal{M}$ while the overall convexity of the cost function is preserved.  To
  see this, consider the least-squares loss regularized by the PMC penalty:
  \begin{align}
    &\frac{1}{2} \|\bm{M} \cdot - \bm{y}\|_2^2 + \mu \underbrace{[\|\cdot\|_1 - \hspace{.3em}^{\tau}
        \|\cdot\|_1 (P_{\mathcal{M}} \cdot )]}_{= \Psi_{\tau, \mathcal{M}}^{\mathrm{PMC}}} \nonumber \\
    & = \underbrace{\frac{1}{2} \|\bm{M} \cdot - \bm{y}\|_2^2 - \mu
        \cdot \hspace{.3em}^{\tau} \|\cdot\|_1 (P_{\mathcal{M}} \cdot )}_{\text{smooth part}} + \mu
      \|\cdot\|_1, \label{eq:PMC_cost}
  \end{align}
  where $\bm{M} \in \mathbb{R}^{m \times n}$ and $\bm{y} \in \mathbb{R}^m$ for some $m, n \in
  \mathbb{N}_*$.  The convexity of \eqref{eq:PMC_cost} is ensured by designing $\mathcal{M}$
  appropriately even for the underdetermined case ($m < n$), as established by the following fact.

  \begin{fact}[{\cite[Proposition 2]{yukawa2023linearly}}]
    \label{fact:PMC_convexity}
    Assume that $\mathcal{M} \coloneqq \Null^{\perp} \bm{M} (= \range \bm{M}^{\top})$.  Then, the smooth
    part of \eqref{eq:PMC_cost}, $ (1/2) \|\bm{M} \cdot - \bm{y}\|_2^2 - \mu \hspace{.3ex}^{\tau}
    \|\cdot\|_1 (P_{\mathcal{M}}\cdot)$, is convex if and only if $\mu \tau^{-1} \le \lambda_{\min}^{++}
    (\bm{M}^{\top} \bm{M})$, where $\lambda_{\min}^{++} (\cdot)$ denotes the smallest strictly positive
    eigenvalue.
  \end{fact}

  While the PMC penalty is additively nonseparable because of the presence of $P_{\mathcal{M}}$, which
  confines the weak convexity to $\mathcal{M} = \Null^{\perp}\bm{M}$, the second term of \eqref{def:PMC}
  can be rewritten as: $\bm{x} \mapsto \hspace{.3ex}^{\tau} \|\cdot\|_1 (P_{\mathcal{M}} \bm{x}) =
  \min_{\bm{u} \in \mathbb{R}^n} [\|\bm{u}\|_1 + (1 / (2 \tau)) \|P_{\mathcal{M}} \bm{x} -
    \bm{u}\|_2^2]$, where the internal objective function is additively separable as a function of
  $\bm{u}$. This structure enables an efficient implementation of the algorithm to minimize
  \eqref{eq:PMC_cost} without extra variables \cite{yukawa2023linearly}.

\section{Main Results}\label{sec:main_results}

In this section, the PMC penalty is incorporated as a regularizer into the LSTD formulation, and the
resulting problem is reformulated as finding the zeros of the sum of a monotone Lipschitz operator and a
hypomonotone operator. An FRBS-based optimization algorithm is proposed to solve this problem.

\subsection{Problem formulation}

To identify a sparse weight vector $\bm{w}$, this paper proposes to combine the LSTD formulation in
\eqref{eq:LSTD_formulation} with the PMC penalty in \eqref{def:PMC}, yielding:
\begin{subequations}
  \begin{align}
    \mathrm{find}~ \bm{w} \in \mathbb{R}^n ~\mathrm{s.t.}~ \bm{w} \in T_{\mathrm{PMC}}(\bm{w}), \label{eq:optimization_problem_original}
  \end{align}
  where
  \begin{align}
    &T_{\mathrm{PMC}}: \mathbb{R}^n \rightarrow 2^{\mathbb{R}^n} \nonumber\\
    &\quad \bm{w} \mapsto \argmin_{\bm{u} \in \mathbb{R}^n} \frac{1}{2} \| \bm{\Phi} \bm{u} - (\bm{g} +
    \gamma \bm{\Phi}' \bm{w})\|_2^2 + \mu \Psi_{\tau, \mathcal{M}}^{\mathrm{PMC}}(\bm{u}).
    \label{eq:optimization_problem_original_lower}
  \end{align}
\end{subequations}
Here, $\mu, \tau \in \mathbb{R}_{++}$ and $\mathcal{M}$ is a linear subspace of $\mathbb{R}^n$. When
$\tau \rightarrow + \infty$, the PMC penalty is reduced to the $\ell_1$-norm, the objective function
$(1/2) \| \bm{\Phi} \bm{u} - (\bm{g} + \gamma \bm{\Phi}' \bm{w})\|_2^2 + \mu \Psi_{\tau,
  \mathcal{M}}^{\mathrm{PMC}}(\bm{u})$ in \eqref{eq:optimization_problem_original_lower} w.r.t.\ $\bm{u}$
becomes convex, and the overall problem takes the form of
LARS-TD~\cite{kolter2009regularization}. However, even in this setting, where the objective function in
\eqref{eq:optimization_problem_original_lower} is convex, it is known that the problem
\eqref{eq:optimization_problem_original} w.r.t.\ $\bm{w}$ cannot be cast as a convex minimization
one~\cite{kolter2009regularization}.

Nevertheless, the objective function in \eqref{eq:optimization_problem_original}
w.r.t.\ $\bm{u}$ can be made convex, for any fixed $\bm{w}$, by a suitable choice of $\mu$ and $\tau$. By
\cref{fact:PMC_convexity}, this convexity is preserved even when $\bm{\Phi}^{\top}
\bm{\Phi}$ is singular, unlike the case with the MC penalty. Exploiting this induced convexity of the objective function in \eqref{eq:optimization_problem_original_lower}, \eqref{eq:optimization_problem_original} can be reformulated as the following non-monotone inclusion
problem, as shown in \cref{PROP:REFORMULATION} below.

\begin{problem}
  \label{problem:PMC_RL}
  Let $\mu, \tau \in \mathbb{R}_{++}$ and let $\mathcal{M}$ be a linear subspace of $\mathbb{R}^n$. Find
  $\bm{w} \in \mathbb{R}^n$ such that
  \begin{align}
    \bm{0}_n \in {} & {} T(\bm{w}) + \mu \partial \| \cdot \|_1 (\bm{w})
    \,, \label{eq:optimization_problem_reformulated}
  \end{align}
  where $T \colon \mathbb{R}^n \rightarrow \mathbb{R}^n \colon \bm{w} \mapsto T(\bm{w})$ is defined as
  \begin{align}
    T(\bm{w}) \coloneqq \bm{\Omega} \bm{w} - \bm{b} - \mu \tau^{-1} P_{\mathcal{M}} (\Id -
    \sProx_{\tau \|\cdot\|_1} )(P_{\mathcal{M}} \bm{w}) \,, \label{map.T}
  \end{align}
  with $\bm{\Omega}$ and $\bm{b}$ as in \eqref{def:A_tilde_b_tilde}.
\end{problem}

\begin{proposition}
  \label{PROP:REFORMULATION}
  Let $\bm{\Phi}^{\top} \bm{\Phi} = \bm{V} \bm{\Lambda} \bm{V}^{\top}$ be the eigenvalue decomposition of
  $\bm{\Phi}^{\top} \bm{\Phi}$, where $\bm{V} \in \mathbb{R}^{n \times n}$ is an orthogonal matrix, and
  $\bm{\Lambda} = \diag (\lambda_1, \lambda_2, \ldots, \lambda_n)$, with $\lambda_1 \ge \lambda_2 \ge
  \ldots \ge \lambda_n \ge 0$.  For some $q \in \overline{1, \rank(\bm{\Phi}^{\top} \bm{\Phi})}$, let
  $\bm{V} = [ \bm{V}_{\overline{1,q}}, \bm{V}_{\overline{q+1, n}}]$, where $\bm{V}_{\overline{1,q}} \in
  \mathbb{R}^{n \times q}$ and $\bm{V}_{\overline{q+1, n}} \in \mathbb{R}^{n \times (n - q)}$.  Assume
  that $\mathcal{M} \coloneqq \Span (\bm{V}_{\overline{1,q}})$ and
  \begin{equation}
    \mu \tau^{-1} \le \lambda_q \,. \label{eq:tau_condition_1}
  \end{equation}
  Then, $\bm{w} \in \mathbb{R}^n$ is a solution, if exists, of \eqref{eq:optimization_problem_original}
  if and only if it is a solution of \cref{problem:PMC_RL}.
\end{proposition}

\begin{IEEEproof}
  The proof is given in \cref{appendix:Proof of Proposition PROP:REFORMULATION}.
\end{IEEEproof}

\begin{remark}
  \label{remark:PMC}
  As discussed in \cref{subsec:Projective Minimax Concave Penalty}, the debiasing effect of
  the PMC penalty is confined to the subspace $\mathcal{M}$.  Since $\Span({\bm{V}}_{\overline{1,q}})$ is
  nondecreasing in $q$, the largest subspace to which the debiasing effect can be confined while
  preserving the convexity is $\mathcal{M} \coloneqq \Span (\bm{V}_{\overline{1,q}}) = \Null^{\perp}
  \bm{\Phi}$, which is attained when $q = \rank(\bm{\Phi}^{\top}\bm{\Phi})$, i.e., when $q$ satisfies
  $\lambda_q = \lambda_{\min}^{++} (\bm{\Phi}^{\top} \bm{\Phi})$.  However, this setting of $q$ does not
  always maximize the debiasing effect of the PMC penalty, as demonstrated in
  \cref{subsec:influence_of_hyperparameter_q} below.  Indeed, if $\lambda_{\min}^{++} (\bm{\Phi}^{\top}
  \bm{\Phi})$ is very small, $\tau$ must be sufficiently large to meet \eqref{eq:tau_condition_1}, which
  can degrade performance since the PMC penalty approaches the $\ell_1$ norm for large $\tau$.  This
  motivates treating $q$ as a tunable hyperparameter in \cref{sec:numerical_tests} unlike the original
  definition of the PMC penalty in \cite{yukawa2023linearly}, where $\mathcal{M}$ is fixed to
  $\Null^{\perp} \bm{\Phi}$.
\end{remark}

\begin{remark}
  One may consider the following alternative formulation instead of
  \eqref{eq:optimization_problem_original}:
  \begin{equation}
    \min_{\bm{w} \in \mathbb{R}^n} \frac{1}{2} \| \bm{\Phi} \bm{w} - (\bm{g} + \gamma \bm{\Phi}'
    \bm{w})\|_2^2 + \mu \Psi_{\tau,
      \mathcal{M}}^{\mathrm{PMC}}(\bm{w}), \label{eq:optimization_problem_Bellman_residual}
  \end{equation}
  the first term of which corresponds to a sample approximation of the mean squared TD error
  \cite[Section 11.5]{sutton1998reinforcement}. Unlike \eqref{eq:optimization_problem_original}, the
  cost function of \eqref{eq:optimization_problem_Bellman_residual} is convex for appropriately chosen
  hyperparameters.  However, it is known that the mean squared TD error without double sampling is not an
  unbiased estimate of the squared $\ell_2$ norm of the Bellman residual
  \cite{sutton1998reinforcement,antos2008learning}.  Given that the primary objective of this study is to
  mitigate estimation bias, this study focuses on the fixed-point formulation based on LSTD.
\end{remark}

\subsection{Recasting and solving \cref{problem:PMC_RL}}\label{sec:recast.problem}

Since the symmetric part $(\bm{\Omega} + \bm{\Omega}^{\top}) / 2$ of $\bm{\Omega}$ is not always positive
semidefinite, $\bm{\Omega}$ is non-monotone in general \cite{bauschke2021generalized}.  In this case, the
form of $T$ in \eqref{map.T} renders $T + \mu \partial \|\cdot\|_1$ non-monotone.  As a result, the
popular forward-backward splitting~\cite{mercier1979lectures, combettes2004solving} and
Douglas-Rachford~\cite{lions1979splitting, eckstein1992douglas, combettes2004solving} methods cannot be
applied since the convergence guarantee of these methods rely on the maximal monotonicity of
operators. To surmount this obstacle, \cref{problem:PMC_RL} is recast and solved using the FRBS
method~\cite{malitsky2020forward}. While the convergence guarantees of FRBS are established
in~\cite{malitsky2020forward} only for the case where $T + \mu \partial \|\cdot\|_1$ is monotone, these
conditions are extended here to accommodate the non-monotone case (see \cref{subsec: Generalized Case}).

To this end, \eqref{eq:optimization_problem_reformulated} is recast as
\begin{align}
  \bm{0}_n \in (\alpha T + \Id) (\bm{w}) + (\alpha \mu \partial \|\cdot\|_1 - \Id)
  (\bm{w}) \,, \label{eq:reforemulated_problem}
\end{align}
where $\alpha \in (0, (\| \bm{\Omega} \|_2 + \mu \tau^{-1} )^{-1} ]$. Then, the following lemma holds.
\begin{lemma}
  The operator $\alpha T + \Id$
  is $\beta$-Lipschitz continuous and maximally monotone for $\beta \coloneqq \alpha (\| \bm{\Omega} \|_2 + \mu \tau^{-1}) + 1$.
\end{lemma}

\begin{IEEEproof}
  The proof is given in \cref{appendix:Lipschitz continuity and maximal monotonicity of T}.
\end{IEEEproof}

  Based on this lemma, it can be shown that
  $\alpha \mu \partial \|\cdot\|_1 - \Id$ is maximally $(-1)$-monotone since $(\alpha \mu \partial \|\cdot\|_1 - \Id) -
  (-1)\Id = \alpha \mu \partial \|\cdot\|_1$ is maximally monotone.

\begin{algorithm}[t!]
\caption{The FRBS method for solving \eqref{eq:optimization_problem_original}} \label{alg:forward_reflected_backward}

\begin{algorithmic}[1]

      \Require $(\bm{\Omega}, \bm{b})$ from \eqref{def:A_tilde_b_tilde}, $\mu \in \mathbb{R}_{++}$, $q
      \in \overline{1, \rank(\bm{\Phi}^{\top} \bm{\Phi})}$, and $\mathcal{M} \coloneqq
      \Span (\bm{V}_{\overline{1, q}})$ (see Proposition~\ref{PROP:REFORMULATION})

      \State Set $\tau$, $\alpha$, and $(\eta_k)_{k \in \mathbb{N}}$ by Proposition~\ref{PROP:CONVERGENCE}

      \State Initialize $\bm{w}_{-1}, \bm{w}_0 \in \mathbb{R}^n$ such that $\bm{w}_{-1} = \bm{w}_0$

      \State Set $\bm{u}_{-1} \coloneqq \alpha (\bm{\Omega} \bm{w}_{-1} - \bm{b} - \mu \tau^{-1}
      (P_{\mathcal{M}} \bm{w}_{-1} - P_{\mathcal{M}}\Soft_{\tau}(P_{\mathcal{M}} \bm{w}_{-1}))) +
      \bm{w}_{-1}$ \For{$k = 0, 1, 2, \dots$} \State $\bm{u}_k \coloneqq \alpha (\bm{\Omega} \bm{w}_k -
      \bm{b} - \mu \tau^{-1} (P_{\mathcal{M}} \bm{w}_k - P_{\mathcal{M}}\Soft_{\tau}(P_{\mathcal{M}}
      \bm{w}_k))) + \bm{w}_{k}$

      \State $\bm{w}_{k+1} \coloneqq \Soft_{\alpha \eta_k \mu / (1 - \eta_k)} ((1 - \eta_k)^{-1}(
      \bm{w}_k - \eta_k \bm{u}_k - \eta_{k-1} (\bm{u}_k - \bm{u}_{k-1})))$

      \EndFor

    \end{algorithmic}

\end{algorithm}

\begin{algorithm}[t!]
    \caption{Approximate policy iteration}\label{alg:policy_iteration_PMC}

    \begin{algorithmic}[1]

      \Require $\pi_0 \in \Pi$, $\gamma \in (0, 1)$, $(s_i, a_i, g(s_i, a_i), s_i')_{i= 1}^m$

      \State Define $\bm{b}$ by \eqref{def:A_tilde_b_tilde}

      \For{$k = 0, 1, 2, \dots$}

      \State Define $\bm{\Omega}$ by \eqref{def:A_tilde_b_tilde}

      \State \textbf{Policy Evaluation:}

      \State \hspace{1em} Obtain $\hat{\bm{w}}$ by Algorithm~\ref{alg:forward_reflected_backward}

      \State \hspace{1em} Set $\hat{Q}^{\pi_k} \coloneqq \bm{\Phi} \hat{\bm{w}}$

      \State \textbf{Policy Improvement:}

      \State \hspace{1em} $\pi_{k+1}(s) \coloneqq \argmin_{a \in \mathfrak{A}} \hat{Q}^{\pi_k}(s, a) , ~ \forall s
      \in \mathfrak{S}$

      \EndFor

    \end{algorithmic}

\end{algorithm}

Applying the FRBS method to \eqref{eq:reforemulated_problem} yields Algorithm~\ref{alg:forward_reflected_backward}.
Here, for any $\tau > 0$, the soft-shrinkage operator is defined as
\begin{align}
  &\Soft_{\tau} \coloneqq \sProx_{\tau\| \cdot \|_1} \colon \mathbb{R}^n \rightarrow \mathbb{R}^n \nonumber \\
    &\quad  [x_1, \ldots,
  x_n]^{\top} \mapsto [\soft_{\tau}(x_1), \ldots, \soft_{\tau}(x_n)]^{\top},
\end{align}
where $\soft_{\tau}(x) \coloneqq
\sign(x) \max \{ |x|-\tau, 0 \}$ and $\sign(x) \coloneqq 1$, if $x \ge 0$, while $\sign(x) \coloneqq -1$, if $x < 0$,
$\forall x\in \mathbb{R}$. The following proposition provides convergence guarantees for
Algorithm~\ref{alg:forward_reflected_backward}.

\begin{assumption}
  \label{assump:for_PROP:CONVERGENCE}
  Assume the following:
    \begin{enumerate}
    \renewcommand{\theenumi}{\normalfont{\theassumption(\alph{enumi})}}
    \renewcommand{\labelenumi}{\normalfont{(\alph{enumi})}}
  \item \label{cond:A1-1} $\mu$, $\tau$, and $q$ satisfy \eqref{eq:tau_condition_1}.
  \item \label{cond:A1-3} $\alpha \in (0, (\|\bm{\Omega}\|_2 + \mu \tau^{-1})^{-1}]$.
\item \label{cond:A1-2} $(\eta_k)_{k \in \mathbb{N} \cup \{-1\}} \subseteq (0, (1 -2 \varepsilon) / (2
  (\beta + 1))]$ for some $\varepsilon \in (0, 1/2)$ s.t.\ $\eta_k \le \eta_{k - 1}$ for all $k \in
  \mathbb{N}$ and $\sum_{k \in \mathbb{N}} \eta_k < + \infty$.
  \end{enumerate}
\end{assumption}

\begin{proposition}[Convergence of Algorithm \ref{alg:forward_reflected_backward}] \label{PROP:CONVERGENCE}
  Let $\beta \coloneqq \alpha (\|\bm{\Omega}\|_2 $ $+ \mu \tau^{-1}) + 1$.
  Suppose that the set $K$ of solutions of \eqref{eq:optimization_problem_original} is nonempty.
  Then, for the sequence $(\bm{w}_k)_{k \in \mathbb{N}}$
  produced by Algorithm~\ref{alg:forward_reflected_backward} with arbitrarily chosen $\bm{w}_{-1} = \bm{w}_0 \in \mathbb{R}^n$, the following holds under \cref{assump:for_PROP:CONVERGENCE}.
  \begin{enumerate}
      \renewcommand{\theenumi}{\normalfont{(\roman{enumi})}}
      \renewcommand{\labelenumi}{\normalfont{(\roman{enumi})}}
    \item For all $\epsilon > 0$, there exists $\delta > 0$ such that, if $\bm{w} \in K$ with $\|\bm{w}_0 - \bm{w}\|_2 < \delta$, then $\|\bm{w}_k - \bm{w}\|_2 < \epsilon$, $\forall k \in \mathbb{N}$.
    \item $(\bm{w}_k)_{k \in \mathbb{N}}$ is a convergent sequence.
  \end{enumerate}
  \end{proposition}

\begin{IEEEproof}
  The proof is given in \cref{appendix:Proof of Proposition PROP:CONVERGENCE} of the supplemental material, which is based on the general results
  to be presented in \cref{subsec: Generalized Case}.
\end{IEEEproof}

An approximate policy-iteration (PI) method, based on Algorithm~\ref{alg:forward_reflected_backward}, is
presented in Algorithm~\ref{alg:policy_iteration_PMC}. The following proposition shows that the sequence
$(\pi_{k})_{k \in \mathbb{N}}$ generated by Algorithm~\ref{alg:policy_iteration_PMC} either converges or
oscillates within a region of the policy space, where the suboptimality of the resulting policies is
bounded by the approximation error of Algorithm~\ref{alg:forward_reflected_backward}.

\begin{proposition}
  \label{prop:convergence_API}
  Let $(\pi_k)_{k \in \mathbb{N}}$ and $(\hat{Q}^{\pi_k})_{k \in \mathbb{N}}$ be sequences generated by
  Algorithm~\ref{alg:policy_iteration_PMC}. Suppose that there exists $\delta \in \mathbb{R}_{++}$ s.t.\
  \begin{equation}
    \|\hat{Q}^{\pi_k} - Q^{\pi_k} \|_{\infty} \le \delta.
    \label{eq:Q_function_eror}
  \end{equation}
  Here, $Q^{\pi_k}$ is defined in \eqref{eq:Q_function}.
  Then, it holds that
  \begin{align}
    \limsup_{k \rightarrow + \infty} \| \hat{Q}^{\pi_k} - Q^{*} \|_{\infty} \le \frac{2 \gamma \delta}{(1
      - \gamma)^2} \,,
  \end{align}
  where $Q^{*}$ is defined in \eqref{eq:optimal_Q}.
\end{proposition}

\begin{IEEEproof}
  By the policy-improvement step of Algorithm~\ref{alg:policy_iteration_PMC}, $\pi_{k+1}(s) = \argmin_{a
    \in \mathfrak{A}} \hat{Q}^{\pi_k}(s,a)$ for all $s \in \mathfrak{S}$, and hence $T_{\pi_{k+1}}
  \hat{Q}^{\pi_k} = T_* \hat{Q}^{\pi_k}$, where $T_*$ is defined in
  \eqref{eq:Bellman_optimality_mapping}.  Hence, the assertion follows from \cite[Theorem
    3.1]{lagoudakis2003least} in the same way as the specialization of the same theorem to the
  least-squares policy-iteration algorithm in \cite[Theorem 7.1]{lagoudakis2003least}.
\end{IEEEproof}

\section{Solving a general non-monotone-inclusion problem} \label{subsec: Generalized Case}
\label{sec:convergence_nonmonotone_inclusion}

Let $(\mathcal{H}, \langle \cdot, \cdot \rangle_{\mathcal{H}})$ be a real Hilbert space with the induced
norm $\| \cdot \|_{\mathcal{H}} \coloneqq \langle \cdot, \cdot \rangle_{\mathcal{H}}^{1/2}$.  Consider
the following problem.

\begin{problem}
  \label{prob:main_problem}
  Find $x^\star \in \mathcal{H}$ such that
  \begin{equation}
    0 \in (A + B)(x^\star), \label{eq:general_problem}
  \end{equation}
  where the operators $A$ and $B$ satisfy the following conditions:
  \begin{enumerate}
    \renewcommand{\theenumi}{\normalfont{(\roman{enumi})}}
    \renewcommand{\labelenumi}{\normalfont{(\roman{enumi})}}
    \item \label{cond:P-1} $A\colon \mathcal{H} \rightarrow 2^{\mathcal{H}}$ is maximally
      ($-\rho$)-monotone for some $\rho \in \mathbb{R}_{+}$.
    \item \label{cond:P-2} $B\colon \mathcal{H} \rightarrow \mathcal{H}$ is monotone and
      $L_{B}$-Lipschitz continuous with $L_B > 0$.
    \item \label{cond:P-3} The solution set $(A + B)^{-1} (0) \neq \emptyset$.
  \end{enumerate}
\end{problem}

\cref{prob:main_problem} encompasses \eqref{eq:reforemulated_problem}, as can be readily verified for the
case where $A = \alpha \mu \partial \|\cdot\|_1 - \Id$ with $\rho = 1$, and $B = \alpha T + \Id$ with
$L_B = \beta$---see \cref{sec:recast.problem}.

To solve \cref{prob:main_problem} apply the FRBS method: for the arbitrarily chosen initial points $x_{-1}, x_0 \in
\mathcal{H}$, generate $(x_k)_{k \in \mathbb{N}} \subset \mathcal{H}$ as: $\forall k \in \mathbb{N}$,
\begin{align}
  \!\!x_{k+1} \!\coloneqq \!J_{\eta_k A} [\, x_k - \eta_k B (x_k) - \eta_{k - 1} (\,\! B(x_k) - B (x_{k - 1})\,\! )\,\!
  ]\,, \label{eq:FRBS_algorithm_generalized}
\end{align}
where $(\eta_k)_{k \in \mathbb{N} \cup \{-1\}} \subset \mathbb{R}_{+}$ is a sequence of step sizes, and
$J_{\eta_k A}$ denotes the resolvent of $\eta_k A$. Here, for any $A\colon \mathcal{H} \rightarrow
2^{\mathcal{H}}$, the resolvent of $A$ is defined as $J_A \coloneqq (\Id + A)^{-1}$, where $\Id$ is the
identity operator in $\mathcal{H}$.  The following proposition provides a closed-form expression of the
resolvent of $\eta_k(\alpha \mu \partial \|\cdot\|_1 - \Id)$, which is required to apply the FRBS method
in \eqref{eq:FRBS_algorithm_generalized} to solve \eqref{eq:reforemulated_problem}.

\begin{proposition}
  \label{PROP:RESOLVENT_OF_A}
  Under \cref{assump:for_PROP:CONVERGENCE}, assume further that $\eta_k < 1$ for all $k\in \mathbb{N}
  \cup \{-1\}$. Then, it holds that
  \begin{align}
    J_{\eta_k(\alpha \mu \partial \|\cdot\|_1 - \Id)} = \Soft_{\alpha \mu \eta_k / (1 - \eta_k)} \circ\,
    (1 - \eta_k)^{-1} \Id \,.
  \end{align}
\end{proposition}

\begin{IEEEproof}
  The proof is given in \cref{appendix:Proof of Proposition PROP:RESOLVENT_OF_A} of the supplemental material.
\end{IEEEproof}

\subsection{Convergence analysis: Lyapunov stability and existence of a limit} \label{subsec:Proof of
  Theorem THEOREM:CONVERGENCE_GUARANTEE}

Although the convergence of the FRBS method is established only for $\rho = 0$~\cite{malitsky2020forward}, the following theorem guarantees convergence even when $\rho$ is strictly positive.

\begin{assumption}
  \label{assump:algorithmic_conditions}
  The sequence $(\eta_k)_{k \in \mathbb{N} \cup \{-1\}}$ satisfies:
  \begin{enumerate}
    \renewcommand{\theenumi}{\normalfont{\theassumption(\alph{enumi})}}
    \renewcommand{\labelenumi}{\normalfont{(\alph{enumi})}}
    \item \label{assump:A2-1} $(\eta_k)_{k \in \mathbb{N} \cup \{-1\}} \subset (0, (1 - 2 \varepsilon) /
      (2 ( L_B + \rho))]$ for some $\varepsilon \in (0, 1/2)$.
    \item \label{assump:A2-2} $\eta_k \le \eta_{k - 1}$ for all $k \in \mathbb{N}$.
    \item \label{assump:A2-3} $\sum_{k \in \mathbb{N}} \eta_k < + \infty$.
  \end{enumerate}
\end{assumption}

A practical choice for the step sizes satisfying Assumption~\ref{assump:algorithmic_conditions} is $\eta_k = c / (k + 2)^{\zeta}$, where $\zeta > 1$ and $c \in (0, (1 - 2 \varepsilon) / (2 ( L_B + \rho))]$.

\begin{theorem}
  \label{THEOREM:CONVERGENCE_GUARANTEE}
  Consider \cref{prob:main_problem} and suppose that \cref{assump:algorithmic_conditions} holds.
\color{black}
  Then, for the sequence $(x_k)_{k \in \mathbb{N}} \subset \mathcal{H}$ generated by
    \eqref{eq:FRBS_algorithm_generalized} with arbitrarily fixed initial points $x_{-1}, x_0 \in \mathcal{H}$, the following hold:
    \begin{enumerate}
      \renewcommand{\theenumi}{\normalfont{(\roman{enumi})}}
      \renewcommand{\labelenumi}{\normalfont{(\roman{enumi})}}
      \item $(x_{k})_{k \in \mathbb{N}}$ is bounded.
      \item It holds that 
            \begin{align}
              \sum_{k \in \mathbb{N}} \|x_{k+1}-x_{k}\|_{\mathcal{H}}^2 < +\infty. \label{eq:sum_squared_difference_finite}
            \end{align}
      \item There exists $(h_k)_{k \in \mathbb{N}} \subset \mathbb{R}_{+}$ such that $\sum_{k \in \mathbb{N}} h_k < +\infty$, and 
            \begin{align}
              \|x_{k+1}-x\|_{\mathcal{H}}^2 \leq \|x_k-x\|_{\mathcal{H}}^2 + h_k
            \end{align}
            for any $x \in (A + B)^{-1} (0)$.
            In other words, $(x_k)_{k \in \mathbb{N}}$ is quasi-Fej\'{e}r monotone \cite{bauschke2017convex} with respect to $(A + B)^{-1} (0)$.
    \end{enumerate}
\end{theorem}

\begin{IEEEproof}
  The proof is given in \cref{appendix:Proof of Theorem THEOREM:CONVERGENCE_GUARANTEE}.
\end{IEEEproof}

\cref{proposition:Lyapunov} below states that setting $x_0$ and $x_{-1}$ close enough to any $x \in (A + B)^{-1} (0)$ guarantees that the subsequent iterates remain in a bounded set around $x$.
This implies that the algorithm \eqref{eq:FRBS_algorithm_generalized} is Lyapunov stable \cite{lasalle2012stability}.

\begin{proposition}
  \label{proposition:Lyapunov}
  Consider \cref{prob:main_problem} and suppose that \cref{assump:algorithmic_conditions} holds.  Assume
  further that $x_{-1} = x_0$.  Then, for all $\epsilon > 0$, there exists $\delta > 0$ such that, if
  $x^{\star} \in (A + B)^{-1} (0)$ with $\|x_0 - x^{\star}\|_{\mathcal{H}} < \delta$, then $\|x_k -
  x^{\star}\|_{\mathcal{H}} < \epsilon$, $\forall k \in \mathbb{N}$.
\end{proposition}

\begin{IEEEproof}
  The proof is given in \cref{appendix:Proof of Corollary proposition:Lyapunov} of the supplemental material.
\end{IEEEproof}

Furthermore, under an additional assumption, the following proposition holds.  For a finite-dimensional
Hilbert space $\mathcal{H}$, this assumption is equivalent to locally boundedness of $A$
\cite{RockWets98}, which is crucial for ensuring the boundedness of $(\|x_k - x_{k+1}\|_{\mathcal{H}}
)_{k \in \mathbb{N}}$.

\begin{assumption}\label{assump:boundedness_of_A}
  $A (S)$ is bounded for every bounded set $S \subset \mathcal{H}$.
\end{assumption}

\begin{proposition}
  \label{proposition:convergent_sequence}
  Consider \cref{prob:main_problem} and suppose that \cref{assump:algorithmic_conditions,assump:boundedness_of_A} hold.
  Then, it holds that
  \begin{align}
    \sum_{k \in \mathbb{N}} \|x_{k+1}-x_{k}\|_{\mathcal{H}} < +\infty, \label{eq:sum_difference_finite}
  \end{align}
  and hence $(x_k)_{k \in \mathbb{N}}$ is a convergent sequence.
\end{proposition}

\begin{IEEEproof}
  The proof is given in \cref{appendix:Proof of Corollary proposition:convergent_sequence} of the supplemental material.
\end{IEEEproof}

\cref{proposition:convergent_sequence,proposition:Lyapunov} jointly provide reliability of the algorithm
without additional strong assumptions for the non-monotone inclusion problem.  Specifically, they ensure
that $(x_k)_{k \in \mathbb{N}}$ converges to a point in the vicinity of a solution when $x_0 = x_{-1}$
are initialized appropriately.  In the following section, the exact convergence to a solution is
established in a different approach than this section by introducing different assumptions.

\subsection{Convergence analysis under the weak MVI assumption}

Although the convergence analysis in \cref{subsec:Proof of Theorem THEOREM:CONVERGENCE_GUARANTEE} relies on relatively weak assumptions, the exact convergence of the generated sequence is not guaranteed.
This section presents a convergence analysis based on the weak Minty variational inequality (MVI) \cite{diakonikolas2021efficient,pethick2023escaping} to guarantee the exact convergence to a solution.
Note that the MVI is a weaker assumption than global monotonicity or comonotonicity, and has been utilized in recent studies on non-monotone inclusion problem \cite{evens2025convergence,pethick2023escaping}.

\begin{assumption}
  \label{assump:algorithmic_conditions_WMVI}
  The following three conditions hold:
  \begin{enumerate}
    \renewcommand{\theenumi}{\normalfont{\theassumption(\alph{enumi})}}
    \renewcommand{\labelenumi}{\normalfont{(\alph{enumi})}}
    \item \label{assump:A3-2} $\eta_k \le \eta_{k - 1}$, for all $k \in \mathbb{N}$.
    \item \label{assump:A3-1} The weak Minty variational inequality (MVI) holds for $A + B$,
      \textit{i.e.}, there exists a nonempty set $\mathscr{Z}^* \subseteq \zer(A + B)$ such that for all
      $x^* \in \mathscr{Z}^*$ and some $\varpi \in \mathbb{R}_{++}$,
      \begin{align}
          \langle v, x- x^* \rangle_{\mathcal{H}} \ge -\frac{\varpi}{2} \| v \|_{\mathcal{H}}^2, \quad
          \forall (x, v) \in \gra (A + B) \,.
      \end{align}
    \item \label{assump:A3-3} There exists $\theta \in (0, 1/4)$ such that $(\eta_k)_{k
      \in \mathbb{N} \cup \{-1\}}$ satisfies the following inequality for all $k \in \mathbb{N}$:
      \begin{align}
        - \frac{1}{4} + \frac{3(1 + \eta_k^2 L_B^2)\varpi}{2 \eta_k} + \eta_k^2 L_B^2 \left( 1 + \frac{3
          \varpi}{2 \eta_{k + 1}} \right) \le - \theta.
        \label{eq:assmp_A3-3}
      \end{align}
  \end{enumerate}
\end{assumption}

\begin{theorem}\label{prop:convergence_weak_MVI}
  Consider \cref{prob:main_problem} and suppose that \cref{assump:algorithmic_conditions_WMVI} holds.
  Then, for the sequence $(x_k)_{k \in \mathbb{N}} \subset \mathcal{H}$ generated by
  \eqref{eq:FRBS_algorithm_generalized} with arbitrarily fixed initial points $x_{-1}, x_0 \in
  \mathcal{H}$, the following hold:
    \begin{enumerate}
        \renewcommand{\theenumi}{\normalfont{(\roman{enumi})}}
        \renewcommand{\labelenumi}{\normalfont{(\roman{enumi})}}
      \item $(x_k)_{k \in \mathbb{N}}$ is bounded.
      \item Every weak sequential cluster point of $(x_k)_{k \in \mathbb{N}}$ belongs to $\zer (A+B)$.
      \item Assume also that $\mathscr{Z}^* = \zer (A+B)$. Then, $(x_k)_{k \in \mathbb{N}}$ converges
        weakly to a solution of \eqref{eq:general_problem}.
    \end{enumerate}
\end{theorem}

\begin{IEEEproof}
  The proof is given in \cref{appendix:Proof of Theorem prop:convergence_weak_MVI} of the supplemental material.
\end{IEEEproof}

Theorem \ref{prop:convergence_weak_MVI} implies that the existence of a solution satisfying the weak MVI serves as a key sufficient condition for the convergence guarantee of the FRBS method in the presence of non-monotone operators. Identifying sufficient conditions for the weak MVI to hold in the specific case of \eqref{eq:reforemulated_problem} is left for future work. Nonetheless, this theoretical foundation can be applied to a broad class of optimization problems in which non-monotone operators arise. Note that the step size conditions in \cref{assump:algorithmic_conditions} and \cref{assump:algorithmic_conditions_WMVI} are not simultaneously satisfiable, since Assumption~\ref{assump:A2-3} requires $\lim_{k \to \infty} \eta_k = 0$, while Assumption~\ref{assump:A3-3} implies that $(\eta_k)_{k \in \mathbb{N} \cup \{-1\}}$ is bounded below by a positive constant (see \cref{lemma:eta_k_bounds} in \cref{appendix:Proof of Theorem prop:convergence_weak_MVI}).

An example of $(\eta_k)_{k \in \mathbb{N} \cup \{-1\}}$ which satisfies Assumption~\ref{assump:A3-3} is
shown in the following proposition.

\begin{proposition}
  \label{prop:step_size}
  Let $\Upsilon\colon \mathbb{R}_{++} \to \mathbb{R}\colon \eta \mapsto L_B^2 \eta^2 + 3 \varpi L_B^2
  \eta + 3 \varpi / (2 \eta)$. Let $\eta_* \in \mathbb{R}_{++}$ denote the unique minimizer of
  $\Upsilon$.  Suppose that $\Upsilon(\eta_*) < 1 / 4$, and set $\theta \in (0, 1/4 - \Upsilon(\eta_*))
  \neq \emptyset$.  Then, the equation $\Upsilon(\eta) = 1 / 4 - \theta$ has solutions $\eta_{-} \in (0,
  \eta_*]$ and $\eta_{+} \in [\eta_*, +\infty)$.  Set $\eta_k = \eta, ~ \forall k \in \mathbb{N} \cup
      \{-1\}$ for arbitrarily fixed $\eta \in [\eta_{-}, \eta_{+}]$.  Then, \eqref{eq:assmp_A3-3} is
      satisfied.
\end{proposition}

\begin{IEEEproof}
  The proof is given in \cref{appendix:Proof of Proposition prop:step_size} of the supplemental material.
\end{IEEEproof}

The unique minimizer $\eta_*$ of $\Upsilon$ and the largest possible constant step size $\eta_{+}$ can be
characterized as the roots of the following cubic equations
\begin{align}
  &\Upsilon'(\eta_*) = 0
  ~ \Leftrightarrow ~
  4 L_B^2 \eta_*^3 + 6 \varpi L_B^2 \eta_*^2 - 3 \varpi = 0,
  \label{eq:cubic_eta_*}
\end{align}
and
\begin{align}
  &\Upsilon(\eta_+) = 1 / 4 - \theta \nonumber \\ \Leftrightarrow ~ &2 L_B^2 \eta_{+}^3 + 6 \varpi L_B^2
  \eta_{+}^2 - (1 / 2 - 2 \theta) \eta_{+} + 3 \varpi = 0,
  \label{eq:cubic_eta_+}
\end{align}
respectively. Hence, the closed-form expressions of $\eta_*$ and $\eta_{+}$ are obtained by eigenvalue
decompositions of the companion matrices associated with \eqref{eq:cubic_eta_*} and
\eqref{eq:cubic_eta_+}, respectively \cite{niu2003method}.

\begin{remark}\label{remark:relations_monotone_FRBS}
  Variants of the FRBS method for treating a non-monotone operator have also been proposed recently. A
  stochastic variance-reduced FRBS~\cite{tran2025variance} attains state-of-the-art oracle complexity
  under a weak Minty condition, and \cite{van2025forward} establishes convergence for a deterministic
  FRBS-type algorithm with anchoring and double inertia. However, \cite{tran2025variance} assumes that
  the set-valued operator $A$ is monotone and Lipschitz and the single-valued operator $B$ is
  non-monotone, targeting the stochastic setting. Furthermore, \cite{van2025forward} assumes that $A + B$
  is monotone while $A$ or $B$ can be hypomonotone. In contrast, this paper considers the unmodified FRBS
  method \eqref{eq:FRBS_algorithm_generalized}, which does not require monotonicity of $A+B$; it assumes
  instead that $A$ is hypomonotone and $B$ is monotone and Lipschitz, targeting the deterministic
  setting.
\end{remark}

\section{Numerical Tests}\label{sec:numerical_tests}

First, three classical benchmark RL tasks, the $50$-state chain walk \cite{lagoudakis2003least}, the
mountain car~\cite{sutton1998reinforcement}, and the acrobot~\cite{sutton1998reinforcement} are used to
validate the effectiveness of the proposed method.  Throughout the experiments, the step sizes
$(\eta_k)_{k \in \mathbb{N} \cup \{-1\}}$ are selected to satisfy \cref{assump:for_PROP:CONVERGENCE}.
The proposed method is compared against LSTD~\cite{lagoudakis2003least},
LARS-TD~\cite{kolter2009regularization}, and the state-of-the-art basis pursuit denoising (BPDN)
approach~\cite{qin2014sparse}.  For LSTD, when $\bm{\Omega}$ is singular, the Moore-Penrose pseudoinverse
of $\bm{\Omega}$ is used instead.  Similarly, the computation of $(\bm{\Phi}^{\top} \bm{\Phi})^{-1}$ used
in BPDN is replaced with the Moore-Penrose pseudoinverse if necessary.  To find the optimal policy, each
method is modified to learn the Q-function so that API is applied.  Finally, the influence of
hyperparameter $q$ on the performance of the proposed method is evaluated.

In all tasks, the $n \times 1$ feature vector is defined as follows: for all $s \in \mathfrak{S}$ and the $j$th action $a^{(j)} \in \mathfrak{A} \coloneqq \{a^{(1)}, a^{(2)}, \ldots, a^{(|\mathfrak{A}|)}\}$,
\begin{equation}
  \bm{\phi}(s, a^{(j)}) \coloneqq
  \bigl[\bm{0}_{(j-1)d}^{\top},\; \bm{\varphi}^{\top}(s),\; \bm{\varepsilon}^{\top},\; \bm{0}_{(|\mathfrak{A}|-j)d}^{\top}\bigr]^{\top} \in \mathbb{R}^n,
  \label{eq:feature_vector}
\end{equation}
where $\bm{\varphi}(s) \coloneqq [1, \varphi_1^{\mathrm{RBF}}(s), \ldots,
  \varphi_{n_{\mathrm{RBF}}}^{\mathrm{RBF}}(s)]$, the Gaussian kernel $\varphi_i^{\mathrm{RBF}}(s)
\coloneqq \exp (\, -(s - c_i)^2 / \varsigma\, )$, $i \in \overline{ 1, n_{\mathrm{RBF}} }$, with centers
$\{ c_i \}_{i=1}^{n_{\mathrm{RBF}}}$ and $\varsigma > 0$, and $\bm{\varepsilon} \in
\mathbb{R}^{n_{\varepsilon}}$ is the realization of a vector-valued random variable, representing
irrelevant features, that is independently drawn for each sample from the normal distribution
$\mathcal{N}(\bm{0}_{n_{\varepsilon}}, 0.1 \bm{I}_{n_{\varepsilon}})$.  Here, $n_{\mathrm{RBF}},
n_{\varepsilon} \in \mathbb{N}_{*}$ s.t.\ $|\mathfrak{A}| d = n$ with $d \coloneqq 1 + n_{\mathrm{RBF}} +
n_{\varepsilon}$.

\subsection{$50$-state chain walk}\label{subsec:chain_walk}

In the $50$-state chain walk~\cite{lagoudakis2003least}, the MDP consists of $50$ states, two actions,
``left'' and ``right'', and the success probability of either action is $0.9$.  When an action fails, the
state changes to the opposite direction.  A reward is $1$ (one-step loss is $-1$) only at states $10$ and
$41$ and zero everywhere else.  The discount factor $\gamma$ is set to $0.9$.  The optimal policy is to
go right in states $1$--$9$ and $26$--$41$, and to go left in states $11$--$25$ and $42$--$50$.  The
cost-to-go function for any policy $\pi$ is defined as $\mathcal{J}^{\pi} \colon \mathfrak{S} \to
\mathbb{R} \colon s \mapsto Q^{\pi}(s, \pi(s))$.  The optimal cost-to-go function is defined as
$\mathcal{J}^* \colon \mathfrak{S} \to \mathbb{R} \colon s \mapsto \min_{\pi} \mathcal{J}^{\pi}(s) =
\min_{a \in \mathfrak{A}} Q^*(s, a)$ \cite{bertsekas2019reinforcement}.  For this MDP, the optimal
cost-to-go function $\mathcal{J}^*$ can be easily obtained analytically by solving
\eqref{eq:Bellman_optimality_equation} because $\mathcal{P}$ is available.  For the feature vector, the
centers $\{c_i\}_{i=1}^{n_{\mathrm{RBF}}}$ are evenly spaced over $\mathfrak{S}$ with $\varsigma = 20$.

The evaluation metric is the normalized mean square error (NMSE) between the true cost-to-go function and
its estimate, defined as
\begin{equation}
  \sum_{s=1}^{50}
[\mathcal{J}^*(s) - \min_{a \in \mathfrak{A}} \hat{Q}(s, a)]^2 / \sum_{s=1}^{50} (\mathcal{J}^*(s))^2,
\end{equation}
where $\hat{Q}$ is an estimate of $Q^*$.  All hyperparameters are carefully tuned for each method to
achieve best performance.

\begin{figure}[t!]
  \centering
  \includegraphics[scale=0.9]{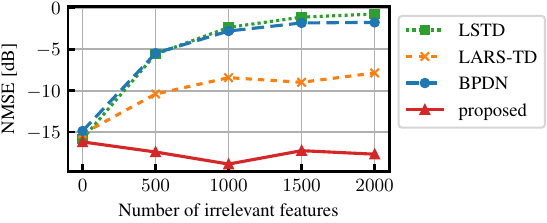}
  \caption{$50$-state chain walk: NMSE vs.\ number of irrelevant features for $2000$ samples.}
  \label{fig:demo_chainwalk_vallue_NMSE_noise}
\end{figure}

\begin{figure}[t!]
  \centering
  \includegraphics[scale=0.9]{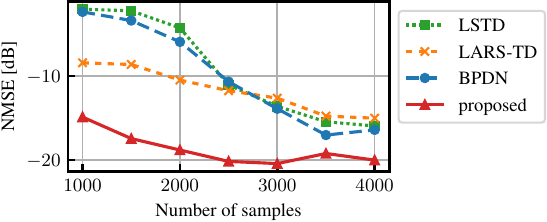}
  \caption{$50$-state chain walk: NMSE vs.\ number of samples for $1000$ irrelevant features.}
  \label{fig:demo_chainwalk_vallue_NMSE_samples}
\end{figure}

Figures~\ref{fig:demo_chainwalk_vallue_NMSE_noise} and~\ref{fig:demo_chainwalk_vallue_NMSE_samples}
depict NMSE as a function of the number of irrelevant features and the number of samples,
respectively. The reported results are averaged over $30$ trials.  For BPDN, the step size required to
guarantee convergence is upper bounded by $1 / \lambda_{\max} (\bm{C}^{\top} \bm{C})$, where $\bm{C}
\coloneqq \gamma \bm{\Phi} (\bm{\Phi}^{\top} \bm{\Phi})^{-1} \bm{\Phi}^{\top} \bm{\Phi}' - \bm{\Phi}$
\cite{qin2014sparse}.  Since many of the features are redundant to each other even without irrelevant
features, $\lambda_{\max} (\bm{C}^{\top} \bm{C})$ becomes significantly large.  Because of this small
step size, the algorithm converges too slowly within a practical number of iterations.  As a result, BPDN
terminates at a dense solution which is similar to the LSTD estimate.

While LARS-TD yields more accurate cost-to-go-function estimates than LSTD and BPDN in the presence of
many irrelevant features, its performance remains limited due to the estimation bias induced by the
$\ell_1$-norm penalty. In contrast, the proposed method significantly outperforms all competitors by
effectively mitigating this bias. Furthermore, it maintains excellent accuracy even when the number of
samples is small.

\subsection{Mountain car}\label{sec:mountain_car}

The ``mountain car''~\cite{sutton1998reinforcement} is a classical task in which the agent accelerates a car to reach the top of the hill from the bottom of a sinusoidal mountain.
The state $\bm{s} \coloneqq [x, v]^{\top} \in \mathfrak{S} \coloneqq [-1.2, 0.5] \times [-0.07, 0.07]$ is continuous, where $x$ and $v$ represent the car's position and its velocity, respectively.
The state transition follows
\begin{align}
  v_{t+1} &\coloneqq \clip(v_t + 0.001\,a_t - 0.0025\cos(3 x_t),\,-0.07,\,0.07), \\
  x_{t+1} &\coloneqq \clip(x_t + v_{t+1},\,-1.2,\,0.5)
\end{align}
with $v_{t+1}$ reset to zero when the car reaches the left boundary ($x_{t+1} = -1.2$, $v_{t+1} < 0$),
and three discrete actions $a_t \in \mathfrak{A} \coloneqq \{-1, 0, 1\}$ are available.  Here, $\clip
\colon \mathbb{R}^3 \rightarrow \mathbb{R} \colon (z, z_{\min}, z_{\max}) \mapsto \min \{ \max \{z,
z_{\min}\}, z_{\max}\}$. The one-step loss is $g(\bm{s}, a) = 0$ at the goal ($x \ge 0.5$) and $g(\bm{s},
a) = 1$ otherwise; the discount factor is $\gamma = 0.99$.  Following \cite{kolter2009regularization},
the centers form two-dimensional grids of $k \times k$ with $k \in \{2, 4, 8, 16, 32\}$, giving
$n_{\mathrm{RBF}} = \sum_{k \in \{2, 4, 8, 16, 32\}} k^2 = 1364$ and, with $n_{\varepsilon} = 500$ and
$|\mathfrak{A}| = 3$, $n = 3 \times (1 + 1364 + 500) = 5595$.

For data collection, a random policy is executed from a state drawn uniformly from $\mathfrak{S}$ with
$50$ episodes.  Each episode is truncated after $10$ steps or upon reaching the goal, yielding at most $m
= 500$ samples.  Approximate policy iteration runs for at most $20$ iterations with stopping criterion
$\|\hat{\bm{w}}_{k+1} - \hat{\bm{w}}_k\|_2 < 10^{-3}$.  The learned policy is evaluated over $10$
episodes of at most $1000$ steps from a random initial state following the uniform distribution over
$[-0.6, -0.4] \times \{0\}$.  A trial is considered successful if the goal is reached within $1000$
steps.

\begin{table}[t!]
    \centering
    \caption{The case of the mountain car.}
    \label{tab:comparison}
    \begin{tabular}{lccc}
        \toprule
        Method & Success Rate (\%) & Average steps & \thead{Number of selected \\ features} \\
        \midrule
        LSTD    & 6.7  & 684.0 $\pm$ 261.0 & 5314.5 $\pm$ 18.5 \\
        LARS-TD   & 60.0 & 171.2 $\pm$ 118.7 & 41.9 $\pm$ 5.1 \\
        BPDN     & 6.7  & 580.5 $\pm$ 245.5 & 925.0 $\pm$ 5.0 \\
        Proposed & \textbf{90.0} & \textbf{151.0 $\pm$ 30.2} & 186.6 $\pm$ 28.1 \\
        \bottomrule
    \end{tabular}
\end{table}

Table~\ref{tab:comparison} summarizes the results, averaged over $30$ trials.  Since the number of
features is significantly larger than the number of samples, the resulting matrices $\bm{\Omega}$ and
$\bm{\Phi}^{\top} \bm{\Phi}$ are severely rank-deficient, leading to a low success rate of LSTD and BPDN.
Although LARS-TD promotes sparsity through the $\ell_1$-norm, it selects too few features, leading to a
low success rate.  These results suggest that the estimation bias inherent in convex regularization is
not negligible in this challenging task.  In contrast, the proposed method achieves the highest success
rate and the fewest average steps to the goal in successful trials among all the methods.  This indicates
that the proposed method effectively mitigates the bias of $\ell_1$-based sparse regularization, yielding
a more accurate Q-function estimate.

\subsection{Acrobot}\label{subsec:acrobot}

The acrobot \cite{sutton1998reinforcement} is an underactuated two-link robot where only the second joint
can exert torque.  The objective is to swing the tip of the second joint above the first joint in as few
steps as possible.  The state $\bm{s} = [\theta_1, \theta_2, \dot{\theta}_1, \dot{\theta}_2]^{\top} \in
\mathfrak{S} \coloneqq [-\pi, \pi]^2 \times [-4\pi, 4\pi] \times [-9\pi, 9\pi]$ is continuous, where
$\theta_1$ and $\theta_2$ are joint angles of the first and second links, respectively.  The action space
is defined as $\mathfrak{A} \coloneqq \{-1, 0, 1\}$, and the state transition follows
\cite{sutton1998reinforcement}.  The one-step loss is $g(\bm{s}, a) = 0$ when the tip of the second link
rises above the level of the first joint, and $1$ otherwise, and the discount factor is $\gamma = 0.99$.
The centers form four-dimensional grids of $k^4$ points with $k \in \{2, 3, 4\}$, giving
$n_{\mathrm{RBF}} = \sum_{k \in \{2, 3, 4\}}k^4 = 353$ and, with $n_{\varepsilon} = 500$ and
$|\mathfrak{A}| = 3$, $n = 3 \times (1 + 353 + 500) = 2562$.

A random policy is executed over $50$ episodes to collect data, each starting from a state drawn
uniformly over $[-\pi, \pi]^2 \times [-2\pi, 2\pi] \times [-4.5\pi, 4.5\pi]$ and truncated after $200$
steps or upon termination, yielding at most $m = 10,000$ samples.  Approximate policy iteration runs for
at most $10$ iterations with stopping criterion $\|\hat{\bm{w}}_{k+1} - \hat{\bm{w}}_k\|_2 < 10^{-5}$.
The learned policy is evaluated over $10$ episodes of at most $1000$ steps from a random initial state
following the uniform distribution over $[-0.1, 0.1]^4$.  A trial is considered successful if the goal is
reached within $500$ steps.

\begin{table}[t!]
    \centering
    \caption{The case of the acrobot.}
    \label{tab:comparison_acrobot}
    \begin{tabular}{lccc}
        \toprule
        Method & Success Rate (\%) & Average steps & \thead{Number of selected \\ features} \\
        \midrule
        LSTD    & 0.0  & --                & -- \\
        LARS-TD   & 63.3 & 333.9 $\pm$ 75.5  & 290.8 $\pm$ 111.7 \\
        BPDN     & 6.7  & 456.1 $\pm$ 39.8  & 2014.5 $\pm$ 0.5 \\
        Proposed & \textbf{86.7} & \textbf{283.8 $\pm$ 95.8}  & 267.0 $\pm$ 65.7 \\
        \bottomrule
    \end{tabular}
\end{table}

Table~\ref{tab:comparison_acrobot} summarizes the results, averaged over $30$ trials.  Since LSTD
achieves a success rate of $0\%$, the remaining metrics are not reported.  BPDN converges to a
comparatively dense solution within this iteration budget and its success rate is low.  While LARS-TD
achieves a moderate success rate by promoting sparsity via the $\ell_1$-norm, its estimation bias still
limits the performance.  In contrast, the proposed method attains both the highest success rate, the
fewest average steps to the goal, and the lowest number of selected features among all the methods.

\subsection{Influence of hyperparameter $q$}
\label{subsec:influence_of_hyperparameter_q}

Figure~\ref{fig:demo_chainwalk_vallue_NMSE_q} shows the influence of the hyperparameter $q$ on the NMSE,
in the same $50$-state chain walk environment with irrelevant features as in \cref{subsec:chain_walk},
with $500$ irrelevant features and $2000$ samples.  The results are averaged over $100$ trials.  The NMSE
attains an almost minimum (about $-18$~dB) for a moderate range $[30, 200]$ of $q$, and degrades outside
of this range.  Since $q$ controls the dimension of the subspace $\mathcal{M}$ onto which the debiasing
effect of the PMC penalty is confined, the influence of estimation bias is large when $q$ is too small.
When $q$ is too large, on the other hand, $\tau$ must be too large to satisfy \eqref{eq:tau_condition_1},
implying that the PMC penalty approaches the $\ell_1$-norm penalty (see \cref{remark:PMC}).

\begin{figure}[t!]
  \centering
  \includegraphics[scale=0.9]{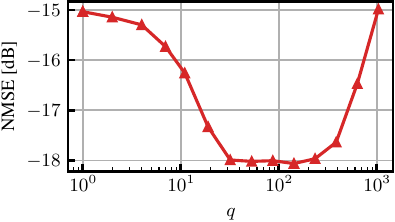}
  \caption{$50$-state chain walk: NMSE vs.\ $q$ for $500$ irrelevant features, $2000$ samples, and $\mu =
    0.5$.}
  \label{fig:demo_chainwalk_vallue_NMSE_q}
\end{figure}

\color{black}

\section{Conclusions}\label{sec:conclusions}

This paper introduced an effective batch algorithm for feature selection in RL, equipped with formal
convergence guarantees. The non-convex PMC penalty was incorporated into the policy-evaluation objective
to promote sparsity in the feature-selection vector. The resulting optimization problem was reformulated
as a non-monotone inclusion, and the FRBS method was applied with newly established convergence
guarantees for a general class of non-monotone inclusion problems subsuming the present
setting. Specifically, Lyapunov stability and the existence of a limit of the iterates under mild
conditions, as well as exact convergence under the weak Minty variational inequality assumption, were
established. Numerical experiments showed that the proposed method substantially outperforms
state-of-the-art approaches, achieving significantly reduced bias.

Future research directions include (i) extending the method to online learning, (ii) integrating it with
the variational framework of~\cite{akiyama2024nonparametric} for Q-function approximation in reproducing
kernel Hilbert spaces (RKHS), (iii) analyzing the error bound \eqref{eq:Q_function_eror} of the
Q-functions for the proposed method (a similar analysis has been conducted for
LASSO-TD~\cite{ghavamzadeh2011finite}, whose algorithmic implementation is equivalent to
LARS-TD~\cite{lagoudakis2003least}), and (iv) identifying sufficient conditions under which the
reformulated problem~\eqref{eq:reforemulated_problem} satisfies the weak MVI.

\appendices
\crefalias{section}{appendix}

\setcounter{equation}{0}
\renewcommand{\theequation}{\Alph{section}.\arabic{equation}}
\setcounter{theorem}{0}
\renewcommand{\thetheorem}{\Alph{section}.\arabic{theorem}}

\section{Proof of \cref{PROP:REFORMULATION}}
\label{appendix:Proof of Proposition PROP:REFORMULATION}

\begin{IEEEproof}[Proof of \cref{PROP:REFORMULATION}]

Problem \eqref{eq:optimization_problem_original} can be reformulated as
  \begin{subequations}\label{eq:optimization_problem_reformulated_1}
    \begin{align}
      &~ \text{find } \bm{w} \in \mathbb{R}^n ~ \mathrm{s.t.} ~&&
    \end{align}
    \begin{align}
      & \bm{u}_* \in \argmin_{\bm{u} \in \mathbb{R}^n}
          \frac{1}{2}\|\bm{\Phi}\bm{u} - (\bm{g} + \gamma \bm{\Phi}' \bm{w})\|_2^2 \!
          + \mu \Psi_{\tau, \mathcal{M}}^{\mathrm{PMC}}(\bm{u}), \!\!\! \label{eq:optimization_problem_reformulated_1_a} \\
      & \bm{w} = \bm{u}_*. \label{eq:optimization_problem_reformulated_1_b}
    \end{align}
  \end{subequations}
  By \eqref{def:PMC}, the objective function in \eqref{eq:optimization_problem_reformulated_1_a} can be
  rewritten as
  \begin{align}
    (1/2)\|\bm{\Phi}\bm{u} - (\bm{g} + \gamma \bm{\Phi}' \bm{w})\|_2^2 + \mu \|\bm{u}\|_1 -
    \mu \cdot \hspace{.3em}^{\tau} \|\cdot\|_1 (P_{
      \mathcal{M} } \bm{u}).
  \end{align}

  Define $\bm{\Lambda}_{\overline{1,q}} \in \mathbb{R}^{q \times q}$ and $\bm{\Lambda}_{\overline{q+1,n}} \in \mathbb{R}^{(n - q) \times (n - q)}$ such that
  \begin{equation}
    \bm{\Lambda} = 
    \begin{bmatrix}
      \bm{\Lambda}_{\overline{1,q}} & \zeromat_{q \times (n - q)}
      \\ \zeromat_{(n-q) \times q} & \bm{\Lambda}_{\overline{q+1,n}}
    \end{bmatrix}.
  \end{equation}
  
    Then, since $\bm{\Phi}^{\top} \bm{\Phi} = \bm{V}_{\overline{1,q}}
  \bm{\Lambda}_{\overline{1,q}} \bm{V}_{\overline{1,q}}^{\top} + \bm{V}_{\overline{q + 1, n}} \bm{\Lambda}_{\overline{q + 1, n}}
  \bm{V}_{\overline{q + 1, n}}^{\top}$, it follows that
  \begin{align}
    \|\bm{\Phi}\bm{u} \|_2^2 &= \langle \bm{u}, \bm{V}_{\overline{1,q}} \bm{\Lambda}_{\overline{1,q}}
    \bm{V}_{\overline{1,q}}^{\top} \bm{u} \rangle_2 + \langle \bm{u}, \bm{V}_{\overline{q + 1, n}}
    \bm{\Lambda}_{\overline{q + 1, n}} \bm{V}_{\overline{q + 1, n}}^{\top} \bm{u} \rangle_2 \nonumber \\ &= \|
    \bm{\Lambda}_{\overline{1,q}}^{1/2} \bm{V}_{\overline{1,q}}^{\top} \bm{u}\|_2^2 + \| \bm{\Lambda}_{\overline{q+1,n}}^{1/2} \bm{V}_{\overline{q + 1, n}}^{\top} \bm{u}\|_2^2. \label{eq:Phi_tilde_u_squared}
  \end{align}

    Let $\bm{v} \coloneqq \bm{g} + \gamma \bm{\Phi}' \bm{w}$ and $f(\bm{u}) \coloneqq \frac{1}{2}\|\bm{\Phi}\bm{u} - \bm{v}\|_2^2 - \mu \cdot \hspace{.3em}^{\tau} \|\cdot\|_1 (P_{\mathcal{M} } \bm{u})$.
Then, by \eqref{eq:Phi_tilde_u_squared}, $f$ can be decomposed as $f = f_1 + f_2$, where
\begin{align}
  f_1(\bm{u}) &\coloneqq \frac{1}{2} \| \bm{\Lambda}_{\overline{1,q}}^{1/2} \bm{V}_{\overline{1,q}}^{\top} \bm{u}\|_2^2 - \mu \cdot \hspace{.3em}^{\tau} \|\cdot\|_1 (P_{\mathcal{M} } \bm{u}), \\
  f_2(\bm{u}) &\coloneqq \frac{1}{2} \| \bm{\Lambda}_{\overline{q+1, n}}^{1/2} \bm{V}_{\overline{q+1, n}}^{\top} \bm{u}\|_2^2 - \langle \bm{\Phi}\bm{u}, \bm{v} \rangle + \frac{1}{2}\|\bm{v}\|_2^2.
\end{align}
Let $\bm{M} \coloneqq \bm{\Lambda}_{\overline{1,q}}^{1/2}\bm{V}_{\overline{1,q}}^\top$.  Then, since
$\lambda_1, \ldots, \lambda_q > 0$ due to $q \le \rank(\bm{\Phi}^\top \bm{\Phi})$, it holds that
$\Null^\perp(\bm M)=\Span(\bm{V}_{\overline{1,q}})=\mathcal{M}$.  Hence, applying
\cref{fact:PMC_convexity} to $\bm M$ and $\bm{y} \coloneqq \bm{0}$ yields that $f_1$ is convex if and
only if $\mu \tau^{-1} \leq \lambda_{\min}^{++}(\bm{M}^{\top}\bm{M}) = \lambda_q$.  Hence, the assumption
in \eqref{eq:tau_condition_1} implies that $f_1$ is convex. Since $f_2$ is clearly convex, their sum $f$
is also convex.

  Therefore, due to Fermat's rule \cite[Theorem 16.3]{bauschke2017convex}, \eqref{eq:optimization_problem_reformulated_1_a} is equivalent to
  \begin{align}
    \bm{0}
    &\in \partial \bigg( \underbrace{\frac{1}{2}\|\bm{\Phi} \cdot - (\bm{g} + \gamma \bm{\Phi}'
      \bm{w})\|_2^2 - \mu \hspace{.3em}^{\tau} \|\cdot\|_1 \circ P_{
      \mathcal{M} } }_{\mathrm{convex}} + \mu \|\cdot\|_1\bigg) (\bm{u}_*)
    \nonumber \\
    &= \nabla \left( \frac{1}{2}\|\bm{\Phi} \cdot - (\bm{g} + \gamma \bm{\Phi}' \bm{w})\|_2^2 -
    \mu \hspace{.3em}^{\tau} \|\cdot\|_1 \circ P_{
      \mathcal{M} } \right)(\bm{u}_*) \nonumber \\
      &\quad + \mu \partial (\|\cdot\|_1)(\bm{u}_*), \label{eq:optimization_problem_reformulated_1_a_nabla_partial_0}
  \end{align}
  where the equality is due to \cite[Corollary 16.48 and Proposition 17.31]{bauschke2017convex} together with the convexity of $(1/2)\|\bm{\Phi} \cdot - (\bm{g} + \gamma \bm{\Phi}'
      \bm{w})\|_2^2 - \mu \hspace{.3em}^{\tau} \|\cdot\|_1 \circ P_{
      \mathcal{M} } $ and $\dom (\mu \|\cdot\|_1) = \mathbb{R}^n$.
  On the other hand, \cref{fact:prox_nabla} together with the chain rule and $P_{\mathcal{M}} = P_{\mathcal{M}}^*$, it holds that \cite{yukawa2023linearly}
  \begin{equation}
    \nabla ({}^{\tau}\|\cdot\|_1 \circ P_{\mathcal{M}})(\bm{u}_*) = \tau^{-1} P_{\mathcal{M}} (\Id -
    \sProx_{\tau \|\cdot\|_1} )(P_{\mathcal{M}} \bm{u}_*). \label{eq:prox_nabla}
  \end{equation}
  Hence, substituting \eqref{eq:prox_nabla} into
  \eqref{eq:optimization_problem_reformulated_1_a_nabla_partial_0} together with the linearity of
  gradient and differentiability of the constituent functions yields that
  \begin{align}
    \bm{0}
    &\in \bm{\Phi}^{\top} (\bm{\Phi}\bm{u}_* - (\bm{g} + \gamma \bm{\Phi}' \bm{w}))
    \nonumber \\
      &\quad - \mu \tau^{-1} P_{\mathcal{M}} (\Id - \sProx_{\tau \|\cdot\|_1} )(P_{\mathcal{M}} \bm{u}_*) + \mu \partial
    \|\cdot\|_1 (\bm{u}_*), \label{eq:optimization_problem_reformulated_1_a_nabla_partial}
  \end{align}
  Combining \eqref{eq:optimization_problem_reformulated_1_b} and
  \eqref{eq:optimization_problem_reformulated_1_a_nabla_partial} completes the proof. 

\end{IEEEproof}

\setcounter{equation}{0}
\renewcommand{\theequation}{\Alph{section}.\arabic{equation}}
\setcounter{theorem}{0}
\renewcommand{\thetheorem}{\Alph{section}.\arabic{theorem}}

\section{On Lipschitz continuity and maximal monotonicity of $\alpha T + \Id$}
\label{appendix:Lipschitz continuity and maximal monotonicity of T}
It is shown below that $\alpha T + \Id$ is $\beta$-Lipschitz continuous and maximally monotone for $\beta \coloneqq \alpha (\|\bm{\Omega}\|_2 + \mu \tau^{-1}) + 1$.

\noindent \textbf{Lipschitz continuity of $\alpha T + \Id$}\\
For any $\bm{x}, \bm{\xi} \in \mathbb{R}^n$, it holds from the triangle inequality that 
\begin{align}
  &\|T (\bm{x}) - T (\bm{\xi})\|_2 \nonumber \\
  &\le  \|\bm{\Omega} \bm{x}  - \bm{\Omega} \bm{\xi}\|_2 + \mu \tau^{-1}\| P_{\mathcal{M}} (\Id - \sProx_{\tau \|\cdot\|_1} )(P_{\mathcal{M}} \bm{x}) \nonumber \\
    &\quad - P_{\mathcal{M}} (\Id - \sProx_{\tau \|\cdot\|_1} )(P_{\mathcal{M}} \bm{\xi}) \|_2. \label{eq:Tx_t_xi}
\end{align}
Since $P_{\mathcal{M}}$ and $\Id - \sProx_{\tau \|\cdot\|_1}$ are nonexpansive \cite[Propositions 4.16 and 12.28]{bauschke2017convex}, it follows that 
\begin{align}
  &\| P_{\mathcal{M}} (\Id - \sProx_{\tau \|\cdot\|_1} )(P_{\mathcal{M}} \bm{x})  \nonumber \\
    &\quad - P_{\mathcal{M}} (\Id - \sProx_{\tau \|\cdot\|_1} )(P_{\mathcal{M}} \bm{\xi}) \|_2
  \le \|\bm{x} - \bm{\xi} \|_2, \label{eq:Nonexpansiveness_P_M_Prox_P_M}
\end{align}
which together with \eqref{eq:Tx_t_xi} show that $T$ is $(\|\bm{\Omega}\|_2 + \mu \tau^{-1})$-Lipschitz continuous. Consequently, $\alpha T + \Id$ is $\beta$-Lipschitz continuous.

\noindent \textbf{Maximal monotonicity of $\alpha T + \Id$}\\
Since $\alpha \in (0, (\|\bm{\Omega}\|_2 + \mu \tau^{-1})^{-1}]$, $\alpha T$ is nonexpansive.
Hence, $\alpha T + \Id$ is monotone by \cite[Example 20.7]{bauschke2017convex}.
Maximality follows from Lipschitz continuity of $\alpha T + \Id$ \cite[Corollary 20.28]{bauschke2017convex}.\hspace{1em plus 1fill}

\setcounter{equation}{0}
\renewcommand{\theequation}{\Alph{section}.\arabic{equation}}
\setcounter{theorem}{0}
\renewcommand{\thetheorem}{\Alph{section}.\arabic{theorem}}

\section{Proof of \cref{THEOREM:CONVERGENCE_GUARANTEE}}
\label{appendix:Proof of Theorem THEOREM:CONVERGENCE_GUARANTEE}

The following fact is used in the proof.
\begin{fact}[{\cite[Lemma 5.31]{bauschke2017convex}}]
  \label{fact:quasi_Fejer}
  Let $(\alpha_k)_{k \in \mathbb{N}}$, $(\beta_k)_{k \in \mathbb{N}}$, $(\gamma_k)_{k \in \mathbb{N}}$, and $(\epsilon_k)_{k \in \mathbb{N}}$ be sequences in $\mathbb{R}_{+}$ such that $\sum_{k \in \mathbb{N}} \gamma_k < + \infty$, $\sum_{k \in \mathbb{N}} \epsilon_k < + \infty$, and
  \begin{align}
    \alpha_{k+1} \le (1 + \gamma_k) \alpha_k - \beta_k + \epsilon_k, ~ \forall k \in \mathbb{N}.
  \end{align}
  Then, $(\alpha_k)_{k \in \mathbb{N}}$ converges and $\sum_{k \in \mathbb{N}} \beta_k < + \infty$.
\end{fact}

\begin{lemma}
  \label{lemma:inequality_from_rho_F_monotonicity}
  Let $F \colon \mathcal{H} \rightarrow 2^{\mathcal{H}}$ be maximally ($-\rho_F$)-monotone for some $\rho_F \in [0, 1)$.
  Set $d_1, u_0, u_1, v_1, v_2 \in \mathcal{H}$ arbitrarily.
  Define $d_2 \coloneqq J_F (d_1 - u_1 - (v_1 - u_0))$.
  Then, for all $x \in \mathcal{H}$ and $u \in - F(x)$, it holds that
  \begin{align}
    &(1 - 2 \rho_F) \|d_2 - x\|_{\mathcal{H}}^2 + 2 \langle v_2 - u_1, x - d_2 \rangle_{\mathcal{H}} + \|d_1 - d_2\|_{\mathcal{H}}^2 \nonumber \\
    &\le \|d_1 - x\|_{\mathcal{H}}^2 + 2 \langle x - d_2, v_1 - u_0 \rangle_{\mathcal{H}} - 2 \langle v_2 - u, d_2 - x \rangle_{\mathcal{H}}. \label{eq:F_rho_monotonicity}
  \end{align}
\end{lemma}

\begin{IEEEproof}
  Since $\rho_F \in [0, 1)$, $J_F$ is single-valued and $\dom J_F = \mathcal{H}$ by \cite[Proposition 2.14]{bauschke2021generalized}, and hence $d_2$ is well-defined.
  It holds by assumption that 
  \begin{align}
    -u + \rho_F x \in (F + \rho_F \Id) (x). \label{eq:monotone_1}
  \end{align}
  On the other hand, it holds by definition of the resolvent that
  \begin{align}
    &d_1 - u_1 - (v_1 - u_0) \in d_2 + F(d_2) \nonumber \\
    \Leftrightarrow ~ & d_1 - u_1 - v_1 + u_0 + (\rho_F - 1) d_2 \in  (F + \rho_F \Id) (d_2) \label{eq:monotone_2}
  \end{align}
  Hence, monotonicity of $F + \rho_F \Id$ together with \eqref{eq:monotone_1} and \eqref{eq:monotone_2} yields that
  \begin{align}
    &0 \le \langle x - d_2,-u + \rho_F x - d_1 + u_1 + v_1 - u_0 + (1 - \rho_F) d_2 \rangle_{\mathcal{H}},
  \end{align}
  which is equivalent to
    \begin{align}
    &- \rho_F \|x - d_2\|_{\mathcal{H}}^2 \le \langle x - d_2, d_2 - d_1 \rangle_{\mathcal{H}} + \langle x- d_2, u_1 - u \rangle_{\mathcal{H}} \nonumber \\
      &\quad + \langle x - d_2, v_1 - u_0 \rangle_{\mathcal{H}}. \label{eq:F_rho_monotonicity_0}
  \end{align}
  The first term of the right-hand side of \eqref{eq:F_rho_monotonicity_0} can be expressed as
  \begin{align}
     \langle x - d_2, d_2 - d_1 \rangle_{\mathcal{H}} = \dfrac{1}{2} \left( \|d_1 - x\|_{\mathcal{H}}^2 - \|d_2 - x\|_{\mathcal{H}}^2 - \|d_2 - d_1\|_{\mathcal{H}}^2 \right).
  \end{align}
  On the other hand, the second term of \eqref{eq:F_rho_monotonicity_0} can be expressed respectively as
  \begin{align}
     \langle x- d_2, u_1 - u \rangle_{\mathcal{H}} &=  \langle x- d_2, v_2 - u \rangle_{\mathcal{H}} +  \langle x- d_2, u_1 - v_2 \rangle_{\mathcal{H}}.
  \end{align}
  Hence, combining these equations yields \eqref{eq:F_rho_monotonicity}.
\end{IEEEproof}

\begin{lemma}
  \label{lemma:p_k_q_k_zeta_k}
  For all $k \in \mathbb{N}$, define
  \begin{align}
    &p_k \coloneqq 1 -  L_B (\eta_k + \eta_{k-1}) - 2 \eta_k \rho,  \label{eq:p_k}\\
    &q_k \coloneqq 1/2 + L_B (\eta_{k-1} - \eta_k) / 2 + \eta_k \rho, \label{eq:q_k} \\
    &\zeta_k \coloneqq p_k^{-1} - 1. \label{eq:zeta_k}
  \end{align}
  Then, the following holds for all $k \in \mathbb{N}$:
  \begin{enumerate}
    \renewcommand{\theenumi}{\normalfont{(\roman{enumi})}}
    \renewcommand{\labelenumi}{\normalfont{(\roman{enumi})}}
    \item \label{eq:p_k_range} $p_k \in [2 \varepsilon, 1)$.  \\
    \item \label{eq:q_k_range} $q_k \ge 1 / 2$, and hence $p_k^{-1} q_k \ge  1 / 2$. \\
    \item \label{eq:zeta_k_range} $\sum_{k=1}^{+\infty} \zeta_k < +\infty$. 
  \end{enumerate}
\end{lemma}

\begin{IEEEproof}
  \begin{enumerate}
    \renewcommand{\theenumi}{\normalfont{(\roman{enumi})}}
    \renewcommand{\labelenumi}{\normalfont{(\roman{enumi})}}
    \item By Assumption \ref{assump:A2-2}, it holds for all $k \in \mathbb{N}$ that
      \begin{align}
        p_k
        &\ge 1 - 2 L_B \eta_{k - 1} - 2 \eta_{k-1} \rho \nonumber \\
        &= 1 - 2 (L_B + \rho)\eta_{k - 1} \nonumber \\
        &\ge 1 - 2 (L_B + \rho) \frac{1 - 2 \varepsilon}{2 (L_B + \rho)} \nonumber \\
        &= 2 \varepsilon,
      \end{align}
      where the second inequality is due to Assumption \ref{assump:A2-1}.
      On the other hand, $p_k < 1$ since $\eta_k > 0$ by Assumption \ref{assump:A2-1}.
      \item The first result follows from Assumptions \ref{assump:A2-1} and \ref{assump:A2-2}, and the second result follows from \ref{eq:p_k_range}.
      \item By Assumption \ref{assump:A2-2}, it holds that
        \begin{align}
          \zeta_k 
          &= \frac{1}{1 -  L_B (\eta_k + \eta_{k-1}) - 2 \eta_k \rho} - 1 \nonumber \\
          &\le \frac{1}{1 -  2 L_B \eta_{k-1} - 2 \eta_{k - 1} \rho} - 1 \nonumber \\
          &= \frac{2 (L_B + \rho)\eta_{k - 1}}{1 \!- \! 2 (L_B + \rho) \eta_{k-1} } \nonumber \\
          &\le \frac{2 (L_B + \rho)\eta_{k - 1}}{1 -  (1 - 2 \varepsilon) } \nonumber \\
          &= \frac{L_B + \rho}{\varepsilon} \eta_{k - 1}, \label{eq:zeta_k_finite}
        \end{align}
        where the second inequality holds due to Assumption \ref{assump:A2-1}.
        The assertion follows from \eqref{eq:zeta_k_finite} and Assumption \ref{assump:A2-3}.
  \end{enumerate}
\end{IEEEproof}

\begin{IEEEproof}[Proof of \cref{THEOREM:CONVERGENCE_GUARANTEE}]
  The proof is in two parts.

  \begin{enumerate}[leftmargin=2ex, itemindent=6ex]
    \renewcommand{\theenumi}{\normalfont{(\roman{enumi})}}
    \renewcommand{\labelenumi}{\normalfont{(\roman{enumi})}}
    \item[\normalfont{(i), (ii)}] Fix $x \in (A + B)^{-1} (0)$ arbitrarily so that $u \coloneqq \eta_k
      B(x) \in - \eta_k A (x)$. Then, by applying Lemma \ref{lemma:inequality_from_rho_F_monotonicity}
      with $F \coloneqq \eta_k A$, $d_1 \coloneqq x_k$, $d_2 \coloneqq x_{k+1}$, $u_0 \coloneqq
      \eta_{k-1} B\left(x_{k-1}\right)$, $u_1 \coloneqq \eta_k B\left(x_k\right)$, $v_1 \coloneqq
      \eta_{k-1} B\left(x_k\right)$, $v_2 \coloneqq \eta_k B\left(x_{k+1}\right)$, and $\rho_F \coloneqq
      \eta_k \rho$, it holds that
    \begin{align}
      & (1 \! - \! 2 \eta_k \rho)  \|x_{k+1}\! - \! x\|_{\mathcal{H}}^2+2 \eta_k\langle B(x_{k+1}) \! - \! B(x_k), x \! - \! x_{k+1}\rangle_{\mathcal{H}} \nonumber \\
        &\quad +\|x_{k+1}-x_k\|_{\mathcal{H}}^2  \nonumber \\
      &\leq\|x_k-x\|_{\mathcal{H}}^2 + 2 \eta_{k-1}\langle B(x_k)-B(x_{k-1}), x - x_{k+1} \rangle_{\mathcal{H}} \nonumber \\
        &\quad - 2 \eta_k \langle B(x_{k + 1})-B(x), x_{k + 1} - x \rangle_{\mathcal{H}} \nonumber \\
      & \leq\|x_k-x\|_{\mathcal{H}}^2 + 2 \eta_{k-1}\langle B(x_k)-B(x_{k-1}), x - x_{k+1}
      \rangle_{\mathcal{H}}, \label{eq:inequality_from_Lemma_D1}
    \end{align}
    where the last inequality is due to monotonicity of $B$. Here, it holds that $\rho_{\mathrm{F}} \in
    [0, 1)$ due to $\eta_k \rho \le (1 - 2 \varepsilon) \rho / (2 ( L_B + \rho)) < 1 / 2 < 1$ by
      Assumption \ref{assump:A2-1}.  The Lipschitz continuity of $B$ yields
  \begin{align}
    &\langle B(x_k)-B(x_{k-1}), x - x_{k+1}\rangle_{\mathcal{H}}  \nonumber \\
    &\quad \le L_B \|x_{k} - x_{k - 1} \|_{\mathcal{H}} \|x - x_{k + 1} \|_{\mathcal{H}} \nonumber \\
    &\quad \le (L_B / 2) (\|x_{k} - x_{k - 1} \|_{\mathcal{H}}^2 +  \|x - x_{k + 1} \|_{\mathcal{H}}^2). \label{eq:inequality_from_monotonicity_of_B_0}
  \end{align}
  In the same way, it holds that
  \begin{align}
    &\langle B(x_{k + 1})-B(x_k), x - x_{k+1}\rangle_{\mathcal{H}} \nonumber \\
      &\quad \ge -(L_B / 2) (\|x_{k + 1} - x_k \|_{\mathcal{H}}^2 +  \|x - x_{k + 1} \|_{\mathcal{H}}^2). \label{eq:Lipschitz_B_1}
  \end{align}
  Hence, combining \eqref{eq:inequality_from_Lemma_D1} with \eqref{eq:inequality_from_monotonicity_of_B_0} and \eqref{eq:Lipschitz_B_1} yields
  \begin{align}
    & (1 \! - \! 2 \eta_k \rho)  \|x_{k+1}\! - \! x\|_{\mathcal{H}}^2 - \eta_k L_B (\|x_{k + 1} - x_k \|_{\mathcal{H}}^2 \nonumber \\
        &\quad +  \|x - x_{k + 1} \|_{\mathcal{H}}^2) +\|x_{k+1}-x_k\|_{\mathcal{H}}^2  \nonumber \\
      &\leq\|x_k-x\|^2 + \eta_{k-1} L_B (\|x_{k} - x_{k - 1} \|_{\mathcal{H}}^2 +  \|x - x_{k + 1} \|_{\mathcal{H}}^2),
  \end{align}
  which is equivalent to
  \begin{align}
      &p_k  \|x_{k+1}-x\|_{\mathcal{H}}^2 + (1 - L_B \eta_k)\|x_{k+1}-x_k\|_{\mathcal{H}}^2  \nonumber \\
        &\quad \leq \|x_k-x\|_{\mathcal{H}}^2 + L_B \eta_{k-1} \|x_{k} - x_{k - 1}\|_{\mathcal{H}}^2, \label{eq:inequality_1}
  \end{align}
  where $p_k$ is defined in \eqref{eq:p_k}. Assumption \ref{assump:A2-1} yields that
  \begin{align}
    L_B \eta_{k - 1} \le \frac{(1 - 2 \varepsilon) L_B}{2 ( L_B + \rho)} < \frac{1}{2} \,. \label{eq:inequality_LB_eta}
  \end{align}
  Due to \eqref{eq:inequality_LB_eta}, it follows from \eqref{eq:inequality_1} that 
  \begin{align}
    & p_k \left(  \|x_{k+1}-x\|_{\mathcal{H}}^2 + \frac{1}{2} \|x_{k+1}-x_k\|_{\mathcal{H}}^2  \right) + q_k \|x_{k+1}-x_k\|_{\mathcal{H}}^2  \nonumber \\
        &\quad \leq \|x_k-x\|_{\mathcal{H}}^2 + \frac{1}{2} \|x_{k} - x_{k - 1}\|_{\mathcal{H}}^2,
  \end{align}
  where $q_k$ is defined in \eqref{eq:q_k}.
  Since $p_k > 0$ by \cref{lemma:p_k_q_k_zeta_k}\ref{eq:p_k_range}, it follows that
  \begin{align}
    & \|x_{k+1}-x\|_{\mathcal{H}}^2 + \frac{1}{2} \|x_{k+1}-x_k\|_{\mathcal{H}}^2 +  p_k^{-1} q_k  \|x_{k+1}-x_k\|_{\mathcal{H}}^2  \nonumber \\
        &\quad \leq (1 + \zeta_k) \left(\|x_k-x\|_{\mathcal{H}}^2 + \frac{1}{2}  \|x_{k} - x_{k -
      1}\|_{\mathcal{H}}^2 \right), \label{eq:quasi_Fejer_two_terms}
  \end{align}
  where $\zeta_k$ is defined in \eqref{eq:zeta_k}. Therefore, in light of
  \cref{lemma:p_k_q_k_zeta_k}\ref{eq:zeta_k_range}, applying \cref{fact:quasi_Fejer} to
  \eqref{eq:quasi_Fejer_two_terms} implies that $(\|x_{k+1}-x\|_{\mathcal{H}}^2 + (1 / 2)
  \|x_{k+1}-x_k\|_{\mathcal{H}}^2)_{k \in \mathbb{N}}$ converges, and
  \begin{equation}
    \sum_{k \in \mathbb{N}}  p_k^{-1} q_k  \|x_{k+1}-x_k\|_{\mathcal{H}}^2  <
    +\infty. \label{eq:sum_p_k_q_k_squared_difference_finite}
  \end{equation}
  Hence, $(\|x_{k+1}-x\|_{\mathcal{H}}^2)_{k \in \mathbb{N}}$ is bounded, and thus $(x_{k})_{k \in \mathbb{N}}$ is bounded.
  Furthermore, it holds from \cref{lemma:p_k_q_k_zeta_k}\ref{eq:q_k_range} and \eqref{eq:sum_p_k_q_k_squared_difference_finite} that
  \begin{align}
    \frac{1}{2} \sum_{k \in \mathbb{N}}  \|x_{k+1}-x_k\|_{\mathcal{H}}^2
    &\le \sum_{k \in \mathbb{N}}  p_k^{-1} q_k  \|x_{k+1}-x_k\|_{\mathcal{H}}^2  \nonumber \\
    &< +\infty,
  \end{align}
  which yields \eqref{eq:sum_squared_difference_finite}.

  \item[\normalfont{(iii)}]
    By discarding a nonnegative term $(1 - L_B \eta_k)\|x_{k+1}-x_k\|_{\mathcal{H}}^2$ from \eqref{eq:inequality_1}, it holds that
    \begin{align}
       \|x_{k+1}-x\|_{\mathcal{H}}^2
        &\leq p_k^{-1} (\|x_k-x\|_{\mathcal{H}}^2 + L_B \eta_{k-1} \|x_{k} - x_{k - 1}\|_{\mathcal{H}}^2) \nonumber \\
        &= \|x_k-x\|_{\mathcal{H}}^2 + h_k, \label{eq:quasi_monotone_x_k}
    \end{align}
    where $h_k \coloneqq  \zeta_k \|x_k-x\|_{\mathcal{H}}^2 + p_k^{-1} L_B \eta_{k-1} \|x_{k} - x_{k - 1}\|_{\mathcal{H}}^2$.
    It remains to verify the summability of $(h_k)_{k \in \mathbb{N}}$.
    Boundedness of $(x_k)_{k \in \mathbb{N}}$ and \cref{lemma:p_k_q_k_zeta_k}\ref{eq:zeta_k_range} imply that
    \begin{align}
      \sum_{k \in \mathbb{N}} \zeta_k \|x_k-x\|_{\mathcal{H}}^2
      &< +\infty. \label{eq:sum_difference_finite_xi}
    \end{align}
    Furthermore, it holds by \cref{lemma:p_k_q_k_zeta_k}\ref{eq:p_k_range} and Assumption
    \ref{assump:A2-1} that
    \begin{align}
      p_k^{-1} L_B \eta_{k - 1}
      &< \frac{1}{2\varepsilon} L_B \eta_{k - 1} \nonumber \\
      &\le \frac{1}{2\varepsilon} L_B \frac{1 - 2 \varepsilon}{2 (L_B + \rho)} \nonumber \\
      &= \frac{L_B (1 - 2 \varepsilon)}{4 \varepsilon (L_B + \rho)}, \label{eq:p_k_inverse_L_B_eta_k}
    \end{align}
    from which together with \eqref{eq:sum_squared_difference_finite} it follows that
    \begin{align}
      &\sum_{k \in \mathbb{N}} p_k^{-1} L_B \eta_{k - 1} \|x_{k}-x_{k - 1}\|_{\mathcal{H}}^2 \nonumber \\
      &< \frac{L_B (1 - 2 \varepsilon)}{4 \varepsilon (L_B + \rho)} \sum_{k \in \mathbb{N}} \|x_{k}-x_{k - 1}\|_{\mathcal{H}}^2 \nonumber \\
      &< +\infty. \label{eq:sum_L_B_eta_k_finite}
    \end{align}
    Finally, combining \eqref{eq:sum_difference_finite_xi} and \eqref{eq:sum_L_B_eta_k_finite} yields the
    summability of $(h_k)_{k \in \mathbb{N}}$.
  \end{enumerate}

\end{IEEEproof}

\bibliographystyle{IEEEtran}
\bibliography{draft_bib}

\clearpage
\setcounter{page}{1}

\begin{center}
  {\large \textbf{Supplemental Material to ``Non-Convex Sparse Reinforcement Learning via Non-Monotone Inclusions''}}
\end{center}

In this supplemental material, we refer to the equation numbers from the main text of the paper when necessary.

\setcounter{equation}{0}
\renewcommand{\theequation}{\Alph{section}.\arabic{equation}}
\setcounter{theorem}{0}
\renewcommand{\thetheorem}{\Alph{section}.\arabic{theorem}}

\section{Proof of \cref{PROP:CONVERGENCE}}
\label{appendix:Proof of Proposition PROP:CONVERGENCE}

\cref{PROP:REFORMULATION} in light of \ref{cond:A1-1} implies that $K = (T + \mu \partial
\|\cdot\|_1)^{-1} (0)$.  According to the discussions in \cref{subsec: Generalized Case}, Problem
\eqref{eq:reforemulated_problem} is a special case of \cref{prob:main_problem}.
\cref{PROP:RESOLVENT_OF_A} can be applied since it holds due to \cref{assump:for_PROP:CONVERGENCE} that
\begin{align}
  \eta_k < \frac{1 -2 \epsilon}{2 (\beta + 1)} < 1,~ \forall k\in \mathbb{N} \cup \{-1\}.
\end{align}

It remains to verify that $(\alpha \mu \partial \|\cdot\|_1 - \Id) (S)$ is bounded for every bounded set
$S$.  Let $\bm{t} \in (\alpha \mu \partial \|\cdot\|_1 - \Id) (S)$.  Then, there exists $\bm{u} \in S$
and $\bm{z} \in \partial \|\cdot\|_1 (s)$ such that
\begin{align}
  \bm{t} = \alpha \mu \bm{z} - s.
\end{align}
Since $S$ is bounded, there exists $R \in \mathbb{R}_{++}$ such that $\|\bm{u}\|_2 < R$. Hence, it
follows that
\begin{align}
  \| \bm{t} \|_2 
  &= \|\alpha \mu \bm{z} - \bm{u} \|_2 \nonumber \\
  &\le \alpha \mu \| \bm{z} \|_2 + \| \bm{u} \|_2 \nonumber \\
  &< \alpha \mu \sqrt{n} + R,
\end{align}
where the last inequality is due to $z_i \in [-1,1]$ for all $i \in \overline{1, n}$.
Hence, $(\alpha \mu \partial \|\cdot\|_1 - \Id) (S)$ is bounded.
Therefore, applying \cref{proposition:Lyapunov} and \cref{proposition:convergent_sequence} proves the assertion.
\hspace{1em plus 1fill}

\setcounter{equation}{0}
\renewcommand{\theequation}{\Alph{section}.\arabic{equation}}
\setcounter{theorem}{0}
\renewcommand{\thetheorem}{\Alph{section}.\arabic{theorem}}

\section{Proof of \cref{PROP:RESOLVENT_OF_A}}
\label{appendix:Proof of Proposition PROP:RESOLVENT_OF_A}
  To treat the resolvent of the non-monotone operator, an abstract subdifferential is first introduced.

  \begin{definition}[{\cite[Definition 6.1]{bauschke2021generalized}}]
    An abstract subdifferential of $f: \mathbb{R}^n \rightarrow (- \infty, + \infty]$ at $\bm{x} \in \mathbb{R}^n$, denoted as $\partial_{\#} f(\bm{x})$, is a subset of $\mathbb{R}^n$ which satisfies the following properties:
    \begin{enumerate}
    \renewcommand{\theenumi}{\normalfont{(\roman{enumi})}}
    \renewcommand{\labelenumi}{\normalfont{(\roman{enumi})}}
      \item\label{item:property_of_abstract_subdifferential_1} $\partial_{\#} f (\bm{x}) = \partial f (\bm{x})$ if $f$ is a proper lower semicontinuous convex function;
      \item\label{item:property_of_abstract_subdifferential_2} $\partial_{\#} f (\bm{x}) = \nabla f (\bm{x})$ if $f$ is continuously differentiable;
      \item\label{item:property_of_abstract_subdifferential_3} $\bm{0} \in \partial_{\#} f (\bm{x})$ if $f$ attains a local minimum at $\bm{x} \in \dom f$;
      \item\label{item:property_of_abstract_subdifferential_4} for every $\bm{y} \in \mathbb{R}^n$ and $\rho \in \mathbb{R}$,
        \begin{align}
          \partial_{\#} \left( f + (\rho / 2) \| \cdot - \bm{y} \|_2^2 \right)
          = \partial_{\#} f + \rho (\Id - \bm{y}).
        \end{align}
    \end{enumerate}
  \end{definition}

  The Clarke-Rockafellar subdifferential, Mordukhovich subdifferential, and Fr\'{e}chet subdifferential
  are all instances of the abstract subdifferential
  \ref{item:property_of_abstract_subdifferential_1}-\ref{item:property_of_abstract_subdifferential_4}---see
  \cite{clarke1990optimization,mordukhovich2018variational,RockWets98} for example.

  \begin{IEEEproof}[Proof of \cref{PROP:RESOLVENT_OF_A}]
    Using an abstract subdifferential, it is verified that 
    \begin{align}
      \eta_k(\alpha \mu \partial \|\cdot\|_1 - \Id)
      &= \partial (\eta_k \alpha \mu  \|\cdot\|_1) - \eta_k \Id \nonumber \\
      &= \partial_{\#} (\eta_k \alpha \mu  \|\cdot\|_1) - \eta_k \Id \nonumber \\
      &= \partial_{\#} \left(\eta_k \alpha \mu  \|\cdot\|_1 - (\eta_k / 2) \| \cdot \|_2^2 \right),
    \end{align}
    where the first equality is due to \cite[Proposition 16.6]{bauschke2017convex}.  Here, $\eta_k \alpha
    \mu \|\cdot\|_1 - (\eta_k/ 2) \| \cdot \|_2^2$ is $\eta_k$-weakly convex with $\eta_k < 1$ by
    assumption. Hence, it holds from \cite[Proposition 6.4]{bauschke2021generalized} that
    \begin{align}
      J_{\eta_k(\alpha \mu \partial \|\cdot\|_1 - \Id)}
      &= J_{\partial_{\#} \left(\eta_k \alpha \mu  \|\cdot\|_1 - (\eta_k / 2) \| \cdot \|_2^2 \right)} \nonumber \\
      &= \sProx_{\eta_k \alpha \mu  \|\cdot\|_1 - (\eta_k/ 2) \| \cdot \|_2^2}. \label{eq:J_sprox}
    \end{align}
    Moreover, it follows from \cite[Proposition 2]{yukawa2025monotone} that
    \begin{align}
      &\sProx_{\eta_k \alpha \mu  \|\cdot\|_1 - (\eta_k / 2) \| \cdot \|_2^2} \nonumber \\
      &= \sProx_{\eta_k \alpha \mu  \|\cdot\|_1 / (1 - \eta_k )} \circ (1 - \eta_k)^{-1} \Id. \label{eq:s_prox_weakly}
    \end{align}
    Hence, combining \eqref{eq:J_sprox} and \eqref{eq:s_prox_weakly} completes the proof.
  \end{IEEEproof}

\setcounter{equation}{0}
\renewcommand{\theequation}{\Alph{section}.\arabic{equation}}
\setcounter{theorem}{0}
\renewcommand{\thetheorem}{\Alph{section}.\arabic{theorem}}

\section{Proof of \cref{proposition:Lyapunov}}
\label{appendix:Proof of Corollary proposition:Lyapunov}

\begin{IEEEproof}
  By \eqref{eq:quasi_Fejer_two_terms}, it holds that 
    \begin{align}
      &z_{k+1} \le (1 + \zeta_k) z_k ,
  \end{align}
  where $z_k \coloneqq \|x_k-x\|_{\mathcal{H}}^2  + (1 / 2) \|x_{k}-x_{k - 1}\|_{\mathcal{H}}^2$.
  Hence, it holds that
  \begin{align}
    z_k \le z_0 \prod_{l=0}^{k - 1} (1 + \zeta_l) \le z_0 \prod_{l=0}^{+\infty} (1 + \zeta_l) < +\infty. \label{eq:quasi_Fejer_Lyapunov}
  \end{align}
  Here, $\prod_{l=0}^{+\infty} (1 + \zeta_l) < +\infty$ since $\sum_{l=0}^{+ \infty}\zeta_l < +\infty$  \cite[Exercise 17.10(iii)]{bauschke2017convex}.
  Since $x_0 = x_{-1}$ by assumption, \eqref{eq:quasi_Fejer_Lyapunov} yields
  \begin{align}
    \|x_k-x^{\star}\|_{\mathcal{H}} \le \|x_0-x^{\star}\|_{\mathcal{H}} \sqrt{\prod_{l=0}^{+\infty} (1 + \zeta_l)}.
  \end{align}
  Hence, the assertion follows with $\delta \coloneqq \epsilon / \sqrt{\prod_{l=0}^{+\infty} (1 + \zeta_l)}$.
\end{IEEEproof}

\setcounter{equation}{0}
\renewcommand{\theequation}{\Alph{section}.\arabic{equation}}
\setcounter{theorem}{0}
\renewcommand{\thetheorem}{\Alph{section}.\arabic{theorem}}

\section{Proof of \cref{proposition:convergent_sequence}}
\label{appendix:Proof of Corollary proposition:convergent_sequence}

\begin{IEEEproof}
    Define $a_k
    \coloneqq \eta_k^{-1} (x_k - x_{k+1} - \eta_k B (x_k)  -  \eta_{k - 1} (B (x_k)  -  B (x_{k - 1})) )
    \in A( x_{k+1})$, where the inclusion holds from \eqref{eq:FRBS_algorithm_generalized}.
  Then, it holds that
  \begin{align}
    &\|x_k - x_{k+1}\|_{\mathcal{H}} \nonumber \\
    &= \|\eta_k a_k + \eta_k B (x_k) +  \eta_{k - 1} (B (x_k)  -  B (x_{k - 1})) \|_{\mathcal{H}} \nonumber \\
    & \le \eta_k \| a_k +  B (x_k) \|_{\mathcal{H}} + \eta_{k - 1} \|  B (x_k)  -  B (x_{k - 1})
    \|_{\mathcal{H}} \nonumber \\
    &\le \eta_k (\|a_k \|_{\mathcal{H}} + \| B (x_k) - B (0) \|_{\mathcal{H}}  + \|B (0)\|_{\mathcal{H}}) \nonumber \\
      &\quad + \eta_{k-1} \|B (x_k)  -  B (x_{k - 1})\|_{\mathcal{H}} \nonumber \\
    &\le \eta_k (\|a_k \|_{\mathcal{H}} + L_B \| x_k \|_{\mathcal{H}}  +  \|B (0)\|_{\mathcal{H}}) + \eta_{k-1} L_B \|x_k  - x_{k - 1}\|_{\mathcal{H}}.
  \end{align}
  Since $(x_k)_{k \in \mathbb{N}}$ is bounded by \cref{THEOREM:CONVERGENCE_GUARANTEE}, $(a_k)_{k \in \mathbb{N}} \subset \bigcup_{k \in \mathbb{N}} A(x_{k+1})$ is bounded by assumption.
  Hence, due to $\eta_k \le \eta_{k-1}$, there exists $M_2 > 0$ such that
  \color{black}
  \begin{align}
    \|x_k - x_{k+1}\|_{\mathcal{H}} \le M_2 \eta_{k - 1},
  \end{align}
  from which together with $\sum_{k \in \mathbb{N}} \eta_k < +\infty$ the inequality \eqref{eq:sum_difference_finite} is established.
  Hence, it holds that 
  \begin{align}
    &\lim_{n \rightarrow + \infty} \sum_{k=n}^{\infty} \|x_k - x_{k+1}\|_{\mathcal{H}} \nonumber \\
    &= \lim_{n \rightarrow + \infty} \left(\sum_{k=0}^{\infty} \|x_k - x_{k+1}\|_{\mathcal{H}} - \sum_{k=0}^{n - 1} \|x_k - x_{k+1}\|_{\mathcal{H}} \right) \nonumber \\
    & = 0.
  \end{align}
  Thus, it follows due to the triangle inequality for $m, n \in \mathbb{N}^*$ with $m > n$ that
  \begin{align}
     \lim_{m, n \rightarrow + \infty}\|x_m - x_n\|_{\mathcal{H}} 
     &= \lim_{m, n \rightarrow + \infty}\left\| \sum_{k = n}^{m - 1} (x_k - x_{k+1}) \right\|_{\mathcal{H}} \nonumber \\
     &\le \lim_{m, n \rightarrow + \infty} \sum_{k = n}^{m - 1} \left\| x_k - x_{k+1} \right\|_{\mathcal{H}} \nonumber \\
     &\le \lim_{n \rightarrow + \infty} \sum_{k = n}^{+\infty} \left\| x_k - x_{k+1} \right\|_{\mathcal{H}} \nonumber \\
     & = 0
  \end{align}
  and hence $(x_k)_{k \in \mathbb{N}}$ is a Cauchy sequence.
  The assertion is proved due to completeness of $\mathcal{H}$.
\end{IEEEproof}

\setcounter{equation}{0}
\renewcommand{\theequation}{\Alph{section}.\arabic{equation}}
\setcounter{theorem}{0}
\renewcommand{\thetheorem}{\Alph{section}.\arabic{theorem}}

\section{Proof of Theorem \ref{prop:convergence_weak_MVI}}
\label{appendix:Proof of Theorem prop:convergence_weak_MVI}

The following lemmas are used in the proof.

\begin{lemma}
  \label{lemma:inequality_boundedness}
  Let $A\colon \mathcal{H} \rightarrow 2^{\mathcal{H}}$ be maximally $\rho$-monotone for some $\rho \in \mathbb{R}$.
  Then, $\gra A$ is closed.
\end{lemma}

\begin{IEEEproof}
  By \cite[Proposition 20.38]{bauschke2017convex}, maximal monotonicity of $A + \rho \Id$ implies that
  $\gra (A + \rho \Id)$ is closed.  Define a continuous linear operator $S: \mathcal{H}^2 \rightarrow
  \mathcal{H}^2: (u, v) \rightarrow (u, v + \rho u)$.  Then, $S^{-1} : \mathcal{H}^2 \rightarrow
  \mathcal{H}^2: (u, v) \rightarrow (u, v - \rho u)$, and hence it holds that $\gra A = S^{-1} \gra (A +
  \rho \Id)$.  Since the preimage of a closed set under a continuous map is closed, $\gra A $ is closed.
\end{IEEEproof}

\begin{lemma}
  \label{lemma:eta_k_bounds}
  Suppose that Assumption~\ref{assump:A3-3} holds.
  Then,
  \begin{align}
    \sup_{k \in \mathbb{N}} \eta_k^{-1} < +\infty
    \quad \text{and} \quad
    \sup_{k \in \mathbb{N}} \eta_k^{-1} \eta_{k - 1} < +\infty.
    \label{eq:eta_k_bounds}
  \end{align}
\end{lemma}

\begin{IEEEproof}
  Since $(\eta_k)_{k \in \mathbb{N} \cup \{-1\}} \subset \mathbb{R}_{++}$ and $L_B, \varpi, \theta \in \mathbb{R}_{++}$, \eqref{eq:assmp_A3-3} implies that, for all $k \in \mathbb{N}$,
  \begin{align}
    \frac{3 \varpi}{2 \eta_k} \le \frac{1}{4} - \theta,
    \label{eq:A3-3_rearranged_1}
  \end{align}
  and
  \begin{align}
    \eta_k^2 L_B^2 \le \frac{1}{4} - \theta.
    \label{eq:A3-3_rearranged_2}
  \end{align}
  Rearranging \eqref{eq:A3-3_rearranged_1} yields
  \begin{align}
    \eta_k^{-1} \le \frac{2}{3 \varpi} \left( \frac{1}{4} - \theta \right), ~ \forall k \in \mathbb{N}.
    \label{eq:eta_k_lower_bound}
  \end{align}
  from which the first inequality of \eqref{eq:eta_k_bounds} follows.
  Similarly, rearranging \eqref{eq:A3-3_rearranged_2} yields
  \begin{align}
    \eta_k \le \frac{1}{L_B } \left( \frac{1}{4} - \theta \right)^{1/2}.
    \label{eq:eta_k_ratio_bound_1}
  \end{align}
  By combining \eqref{eq:eta_k_lower_bound} and \eqref{eq:eta_k_ratio_bound_1}, it holds for all $k \in \mathbb{N}$ that
  \begin{align}
    \eta_k^{-1} \eta_{k-1}
    &\le \frac{2}{3 \varpi L_B} \left( \frac{1}{4} - \theta \right)^{3/2},
    \label{eq:eta_k_ratio_bound_2}
  \end{align}
  which establishes the second inequality of \eqref{eq:eta_k_bounds}.
\end{IEEEproof}

\begin{IEEEproof}[Proof of Theorem \ref{prop:convergence_weak_MVI}]
  Fix $k \in \mathbb{N}$ and $x^* \in \mathscr{Z}^*$ arbitrarily.
  \begin{enumerate}
        \renewcommand{\theenumi}{\normalfont{(\roman{enumi})}}
        \renewcommand{\labelenumi}{\normalfont{(\roman{enumi})}}
    \item 
  Let $\Delta x_k \coloneqq x_k - x_{k-1}$ and $\Delta B_k \coloneqq B(x_k) - B(x_{k-1})$.
  Then, it holds by \eqref{eq:FRBS_algorithm_generalized} that
  \begin{align}
    &x_k \!-\! \eta_k B(x_k) \!-\! \eta_{k - 1} (B(x_k) - B(x_{k - 1})) \in (\Id + \eta_k A) (x_{k + 1}) \nonumber \\
    & \Leftrightarrow - \Delta x_{k + 1}  - \eta_{k - 1} \Delta B_k - \eta_k B(x_k)  \in \eta_k A(x_{k + 1}).
    \label{eq:x_k_in_A}
  \end{align}
  Adding $\eta_k B(x_{k+1})$ to both sides of \eqref{eq:x_k_in_A} yields
  \begin{align}
    - \Delta x_{k + 1} + \eta_k \Delta B_{k + 1} - \eta_{k - 1} \Delta B_k \in \eta_k (A + B) (x_{k + 1}).
    \label{eq:Delta_x_k_in_A}
  \end{align}
  Let $w_k \coloneqq \eta_k^{-1} (- \Delta x_{k + 1} + \eta_k \Delta B_{k + 1} - \eta_{k - 1} \Delta B_k)$.
  Then, since $w_k \in (A + B) (x_{k + 1})$, it follows from the weak MVI assumption that
  \begin{align}
    &\langle w_k, x_{k + 1} - x^* \rangle_{\mathcal{H}} \ge - \frac{\varpi}{2} \|w_k\|_{\mathcal{H}}^2. \nonumber \\
    \Leftrightarrow ~ &U_k^{(1)} + U_k^{(2)}
      \ge -  \frac{\eta_k \varpi}{2} \|w_k\|_{\mathcal{H}}^2, \label{eq:two_angles}
  \end{align}
  where $U_k^{(1)} \coloneqq \langle - \Delta x_{k + 1} , x_{k + 1} - x^* \rangle_{\mathcal{H}}$ and $U_k^{(2)} \coloneqq \langle \eta_k \Delta B_{k + 1} - \eta_{k - 1} \Delta B_k, x_{k + 1} - x^* \rangle_{\mathcal{H}}$.
  Since $- \langle \xi_1, \xi_2 \rangle_{\mathcal{H}} \!=\! (1 / 2) \|\xi_1 - \xi_2\|_{\mathcal{H}}^2  -  (1 / 2) \|\xi_1\|_{\mathcal{H}}^2 - (1 / 2) \|\xi_2\|_{\mathcal{H}}^2$ for all $\xi_1, \xi_2 \in \mathcal{H}$, $U_k^{(1)}$ can be rewritten as
  \begin{align}
    U_k^{(1)}
    &= \frac{1}{2} \|\Delta x_{k + 1} - (x_{k + 1} - x^*) \|_{\mathcal{H}}^2 - \frac{1}{2} \|\Delta x_{k + 1}\|_{\mathcal{H}}^2 \nonumber \\
      &\quad- \frac{1}{2} \|x_{k + 1} - x^*\|_{\mathcal{H}}^2 \nonumber \\
    &= \frac{1}{2} \|x_k - x^*\|_{\mathcal{H}}^2 - \frac{1}{2} \|\Delta x_{k + 1}\|_{\mathcal{H}}^2 - \frac{1}{2} \|x_{k + 1} - x^*\|_{\mathcal{H}}^2.
    \label{eq:two_angles_1}
  \end{align}
  On the other hand, $U_k^{(2)}$ can be rewritten as
  \begin{align}
    U_k^{(2)}
    &= \eta_k \langle \Delta B_{k + 1}, x_{k + 1} - x^* \rangle_{\mathcal{H}} - \eta_{k - 1} \langle \Delta B_k, x_k - x^* \rangle_{\mathcal{H}} \nonumber \\
      &\quad - \eta_{k - 1} \langle \Delta B_k, \Delta x_{k + 1} \rangle_{\mathcal{H}}.
      \label{eq:two_angles_2}
  \end{align}
  Let $\mathfrak{L}_k \coloneqq (1 / 2) \| x_k - x^* \|_{\mathcal{H}}^2 - \eta_{k - 1} \langle \Delta B_k, x_k - x^* \rangle_{\mathcal{H}}$.
  Then, it holds from \eqref{eq:two_angles}, \eqref{eq:two_angles_1}, and \eqref{eq:two_angles_2} that
  \begin{align}
    &\mathfrak{L}_{k + 1} - \mathfrak{L}_{k} \nonumber \\
    &= \frac{1}{2} \| x_{k + 1} - x^* \|_{\mathcal{H}}^2 - \eta_k \langle \Delta B_{k + 1}, x_{k + 1} - x^* \rangle_{\mathcal{H}} \nonumber \\
      &\quad - \frac{1}{2} \| x_k - x^* \|_{\mathcal{H}}^2 + \eta_{k - 1} \langle \Delta B_k, x_k - x^* \rangle_{\mathcal{H}} \nonumber \\
    &= - U_k^{(1)} - U_k^{(2)} - \frac{1}{2} \| \Delta x_{k + 1} \|_{\mathcal{H}}^2 - \eta_{k - 1} \langle \Delta B_k, \Delta x_{k + 1} \rangle_{\mathcal{H}} \nonumber \\
    &\le \frac{\eta_k \varpi}{2} \|w_k\|_{\mathcal{H}}^2 - \frac{1}{2} \| \Delta x_{k + 1} \|_{\mathcal{H}}^2 - \eta_{k - 1} \langle \Delta B_k, \Delta x_{k + 1} \rangle_{\mathcal{H}}.
    \label{eq:L_k_difference}
  \end{align}
  Since $- \langle \xi_1, \xi_2 \rangle_{\mathcal{H}} \le  \|\xi_1\|_{\mathcal{H}}  \|\xi_2\|_{\mathcal{H}} = \sqrt{2} \|\xi_1\|_{\mathcal{H}} \cdot ( 1 / \sqrt{2}) \|\xi_2 \|_{\mathcal{H}} \le  \|\xi_1 \|_{\mathcal{H}}^2  + (1 / 4) \|\xi_2\|_{\mathcal{H}}^2$ for all $\xi_1, \xi_2 \in \mathcal{H}$ due to the AM-GM inequality, it holds that
  \begin{align}
    - \eta_{k - 1} \langle \Delta B_k, \Delta x_{k + 1} \rangle_{\mathcal{H}}
    &\le  \eta_{k - 1}^2 \|\Delta B_k\|_{\mathcal{H}}^2 + \frac{1}{4} \|\Delta x_{k + 1}\|_{\mathcal{H}}^2 \nonumber \\
    &\le \eta_{k - 1}^2 L_B^2 \|\Delta x_k\|_{\mathcal{H}}^2 + \frac{1}{4} \|\Delta x_{k + 1}\|_{\mathcal{H}}^2 ,
    \label{eq:array_inequality_negative}
  \end{align}
  where the last inequality is due to the Lipschitz continuity of $B$.
  On the other hand, since $\|(\xi_1 + \xi_2 + \xi_3) / 3\|_{\mathcal{H}}^2 \le (1 / 3)(\|\xi_1\|_{\mathcal{H}}^2 + \| \xi_2\|_{\mathcal{H}}^2 + \| \xi_3\|_{\mathcal{H}}^2) \Leftrightarrow \|\xi_1 + \xi_2 + \xi_3\|_{\mathcal{H}}^2 \le 3(\|\xi_1\|_{\mathcal{H}}^2 + \| \xi_2\|_{\mathcal{H}}^2 + \| \xi_3\|_{\mathcal{H}}^2)$ for all $\xi_1, \xi_2 , \xi_3 \in \mathcal{H}$ due to Jensen's inequality \cite{bertsekas2003convex}, it holds that
  \begin{align}
    \eta_k^2 \|w_k\|_{\mathcal{H}}^2
    &= \| - \Delta x_{k + 1} + \eta_k \Delta B_{k + 1} - \eta_{k - 1} \Delta B_k \|_{\mathcal{H}}^2 \nonumber \\
    &\le 3 (\| \Delta x_{k + 1} \|_{\mathcal{H}}^2 + \|\eta_k \Delta B_{k + 1} \|_{\mathcal{H}}^2 + \|\eta_{k - 1} \Delta B_k \|_{\mathcal{H}}^2) \nonumber \\
    &\le 3(1 + \eta_k^2 L_B^2) \| \Delta x_{k + 1} \|_{\mathcal{H}}^2 + 3\eta_{k - 1}^2 L_B^2 \| \Delta x_k \|_{\mathcal{H}}^2,
    \label{eq:eta_k_w_k_2}
  \end{align}
  where the last inequality is due to the Lipschitz continuity of $B$.
  Hence, it follows from \eqref{eq:eta_k_w_k_2} that
  \begin{align}
    \eta_k \varpi \|w_k\|_{\mathcal{H}}^2
    &\le \frac{3(1 + \eta_k^2 L_B^2)\varpi}{\eta_k} \| \Delta x_{k + 1} \|_{\mathcal{H}}^2 \nonumber \\
      &\quad + \frac{3\eta_{k - 1}^2 L_B^2 \varpi}{\eta_k} \| \Delta x_k \|_{\mathcal{H}}^2.
    \label{eq:eta_k_omega_w_k_2}
  \end{align}
  Combining \eqref{eq:L_k_difference}, \eqref{eq:array_inequality_negative}, and \eqref{eq:eta_k_omega_w_k_2} yields
  \begin{align}
    \mathfrak{L}_{k + 1} - \mathfrak{L}_{k} 
      &\le \underbrace{\left(- \frac{1}{4} + \frac{3(1 + \eta_k^2 L_B^2)\varpi}{2 \eta_k} \right)}_{\eqqcolon V_k^{(1)}} \| \Delta x_{k + 1} \|_{\mathcal{H}}^2 \nonumber \\
      &\quad + \underbrace{\eta_{k - 1}^2 L_B^2 \left( 1 + \frac{3 \varpi}{2 \eta_k} \right)}_{\eqqcolon V_k^{(2)}} \| \Delta x_k \|_{\mathcal{H}}^2.
      \label{eq:L_diff}
  \end{align}
  Let $\tilde{\mathfrak{L}}_k \coloneqq \mathfrak{L}_k + V_k^{(2)} \|\Delta x_k\|_{\mathcal{H}}^2$.
  Then, it holds from \eqref{eq:L_diff} that
  \begin{align}
    \tilde{\mathfrak{L}}_{k + 1} - \tilde{\mathfrak{L}}_{k}
    &= \mathfrak{L}_{k + 1} - \mathfrak{L}_{k} + V_{k + 1}^{(2)} \| \Delta x_{k + 1} \|_{\mathcal{H}}^2 - V_k^{(2)} \| \Delta x_k \|_{\mathcal{H}}^2 \nonumber \\
    &\le (V_k^{(1)} + V_{k + 1}^{(2)}) \| \Delta x_{k + 1}\|_{\mathcal{H}}^2 \nonumber\\
    &\le - \theta  \| \Delta x_{k + 1}\|_{\mathcal{H}}^2,
    \label{eq:L_tilde_diff}
  \end{align}
  where the last inequality is due to Assumption~\ref{assump:A3-3}.
  Hence, it follows from \eqref{eq:L_tilde_diff} that
  \begin{align}
    \| \Delta x_{k + 1}\|_{\mathcal{H}}^2
    \le \theta^{-1} (\tilde{\mathfrak{L}}_{k} - \tilde{\mathfrak{L}}_{k + 1} ).
    \label{eq:Delta_x_k_1_inequality}
  \end{align}
  
  On the other hand, it holds in the same way as \eqref{eq:array_inequality_negative} that
  \begin{align}
    \eta_{k - 1} \langle \Delta B_k, x_k - x^* \rangle_{\mathcal{H}}
    &\le  \eta_{k - 1}^2 L_B^2 \|\Delta x_k\|_{\mathcal{H}}^2 + \frac{1}{4} \|x_k - x^*\|_{\mathcal{H}}^2.
  \end{align}
  Hence, it follows that
  \begin{align}
    \tilde{\mathfrak{L}}_{k}
    &\coloneqq \mathfrak{L}_k + V_k^{(2)} \|\Delta x_k\|_{\mathcal{H}}^2 \nonumber \\
    &= \frac{1}{2} \| x_k - x^* \|_{\mathcal{H}}^2 - \eta_{k - 1} \langle \Delta B_k, x_k - x^* \rangle_{\mathcal{H}} + V_k^{(2)} \|\Delta x_k\|_{\mathcal{H}}^2 \nonumber \\
    &\ge \frac{1}{4} \| x_k - x^* \|_{\mathcal{H}}^2 + (V_k^{(2)} - \eta_{k - 1}^2 L_B^2) \|\Delta x_k\|_{\mathcal{H}}^2 \nonumber \\
    &= \frac{1}{4} \| x_k - x^* \|_{\mathcal{H}}^2 + \frac{3 \varpi \eta_{k - 1}^2 L_B^2}{2 \eta_k} \|\Delta x_k\|_{\mathcal{H}}^2.
    \label{eq:tilde_L_k_inequality}
  \end{align}

  On the other hand, since $\theta > 0$ by Assumption~\ref{assump:A3-3}, \eqref{eq:L_tilde_diff} implies that $(\tilde{\mathfrak{L}}_{k})_{k \in \mathbb{N}}$ is monotonically nonincreasing.
  Combining this monotonicity with \eqref{eq:tilde_L_k_inequality} yields that
  \begin{align}
    \| x_k - x^* \|_{\mathcal{H}}^2
    \le 4 \tilde{\mathfrak{L}}_{k} 
    \le 4 \tilde{\mathfrak{L}}_0, ~ \forall k \in \mathbb{N}.
    \label{eq:upper_bound_x_k_x_star}
  \end{align}
  Hence, it holds that $\|x_k\|_{\mathcal{H}} < \|x^*\|_{\mathcal{H}} + 2\sqrt{\tilde{\mathfrak{L}}_0}$, implying that $(x_k)_{k \in \mathbb{N}}$ is bounded.

  \item 
  It holds from \eqref{eq:tilde_L_k_inequality} that $\tilde{\mathfrak{L}}_{k} \ge 0$.
  Consequently, since $\tilde{\mathfrak{L}}_{k+1}$ is monotonically nonincreasing, $\lim_{k \to +\infty} \tilde{\mathfrak{L}}_k$ exists, and it holds that
  \begin{align}
    0 \le \lim_{k \to +\infty} \tilde{\mathfrak{L}}_k < + \infty.
    \label{eq:tilde_L_k_0}
  \end{align}
  Hence, summing both sides of \eqref{eq:Delta_x_k_1_inequality} for $k \in \mathbb{N}^*$ yields
  \begin{align}
    \sum_{k = 1}^{+\infty} \| \Delta x_{k + 1}\|^2
    &\le \theta^{-1} \sum_{k = 1}^{+\infty} (\tilde{\mathfrak{L}}_{k} - \tilde{\mathfrak{L}}_{k + 1} ) \nonumber \\
    &= \theta^{-1}  (\tilde{\mathfrak{L}}_1 - \lim_{n \rightarrow +\infty} \tilde{\mathfrak{L}}_{n + 1} )\nonumber \\
    &\le \theta^{-1} \tilde{\mathfrak{L}}_1
    < + \infty,
    \label{eq:sum_Delta_x_k}
  \end{align}
  from which it follows that
  \begin{align}
    \lim_{k \to +\infty} \|\Delta x_{k + 1}\|_{\mathcal{H}} = 0.
    \label{eq:Delta_x_k_0}
  \end{align}

  Due to Lipschitz continuity of $B$, \eqref{eq:Delta_x_k_0}, and \cref{lemma:eta_k_bounds}, it follows
  that
  \begin{align}
    &\lim_{k \rightarrow +\infty} \|w_k \|_{\mathcal{H}} \nonumber \\
    &= \lim_{k \rightarrow +\infty} \|\eta_k^{-1} (- \Delta x_{k + 1} + \eta_k \Delta B_{k + 1} - \eta_{k - 1} \Delta B_k)\|_{\mathcal{H}} \nonumber \\
    &\le \lim_{k \rightarrow +\infty} \eta_k^{-1} (1 + \eta_k L_B) \|\Delta x_{k + 1} \|_{\mathcal{H}} + \eta_k^{-1} \eta_{k - 1} L_B \| \Delta x_k\|_{\mathcal{H}} \nonumber \\
    & = 0.
    \label{eq:limit_norm_w}
  \end{align}

  Since $(x_k)_{k \in \mathbb{N}}$ is bounded, it possesses a weakly convergent subsequence \cite[Lemma 2.45]{bauschke2017convex}.
  Let $\bar{x}$ be a sequential weak cluster point of $(x_k)_{k \in \mathbb{N}}$, say $x_{l_k} \rightharpoonup \bar{x}$, and let $a_k \coloneqq w_k - B(x_{k+1}) \in A( x_{k+1})$.
  Then, it holds due to \eqref{eq:limit_norm_w} and the continuity of $B$ and the norm that
  \begin{align}
    0 
    &= \lim_{k \rightarrow +\infty} \|a_k + B (x_{k + 1}) \|_{\mathcal{H}}  \nonumber \\
    &= \lim_{k \rightarrow +\infty} \|a_{l_k} + B (x_{l_k + 1}) \|_{\mathcal{H}} \nonumber \\
    &=  \left\|\lim_{k \rightarrow +\infty} a_{l_k} + B (\bar{x}) \right\|_{\mathcal{H}},
  \end{align}
  and hence
  \begin{align}
    \lim_{k \rightarrow +\infty} a_{l_k} = - B (\bar{x}).
  \end{align}
  By closedness of $\gra A$ due to \cref{lemma:inequality_boundedness}, $\lim_{k \rightarrow +\infty} (x_{l_k + 1}, a_{l_k}) = (\bar{x}, - B (\bar{x}))$ implies that $(\bar{x}, - B (\bar{x})) \in \gra A$, \textit{i.e.}, $0 \in (A + B)(\bar{x})$.
  
  \item By Cauchy-Schwarz inequality, it holds from Assumption \ref{assump:A3-2} that
  \begin{align}
    |\eta_{k - 1} \langle \Delta B_k, x_k - x^* \rangle_{\mathcal{H}}|
    &\le \eta_{k - 1} \|\Delta B_k\|_{\mathcal{H}} \|x_k - x^*\|_{\mathcal{H}} \nonumber \\
    &\le 2 \eta_0 L_B \|\Delta x_k\|_{\mathcal{H}} \sqrt{ \tilde{\mathfrak{L}}_0},
    \label{eq:B_x_k_x_star_angle}
  \end{align}
  where the last inequality is due to the Lipschitz inequality of $B$ and \eqref{eq:upper_bound_x_k_x_star}.
  Hence, in light of \eqref{eq:Delta_x_k_0}, \eqref{eq:B_x_k_x_star_angle} yields that
  \begin{align}
    \lim_{k \to +\infty} \eta_{k - 1} \langle \Delta B_k, x_k - x^* \rangle_{\mathcal{H}} = 0.
    \label{eq:B_x_k_x_star_angle_0}
  \end{align}
  On the other hand, since $(\eta_k)_{k \in \mathbb{N} \cup \{-1\}} \subset \mathbb{R}_{++}$ and $L_B, \varpi, \theta \in \mathbb{R}_{++}$, \eqref{eq:assmp_A3-3} implies that
  \begin{align}
    V_{k + 1}^{(2)} \coloneqq \eta_k^2 L_B^2 \left( 1 + \frac{3 \varpi}{2 \eta_{k + 1}} \right) \le \frac{1}{4} - \theta, ~ \forall k \in \mathbb{N}.
  \end{align}
  Hence, it holds that
  \begin{align}
    V_k^{(2)} \|\Delta x_k\|_{\mathcal{H}}^2 
    \le  \left(\frac{1}{4} - \theta \right) \|\Delta x_k\|_{\mathcal{H}}^2,
  \end{align}
  from which together with \eqref{eq:Delta_x_k_0} it follows that
  \begin{align}
    \lim_{k \to +\infty} V_k^{(2)} \|\Delta x_k\|_{\mathcal{H}}^2 = 0.
    \label{eq:V_k_2_norm_limit}
  \end{align}
  Since it holds that
  \begin{align}
    &\frac{1}{2} \|x_k - x^*\|_{\mathcal{H}}^2 \nonumber \\
    &= \mathfrak{L}_k + \eta_{k - 1} \langle \Delta B_k, x_k - x^* \rangle_{\mathcal{H}} \nonumber \\
    &= \tilde{\mathfrak{L}}_k - V_k^{(2)} \|\Delta x_k\|_{\mathcal{H}}^2 + \eta_{k - 1} \langle \Delta B_k, x_k - x^* \rangle_{\mathcal{H}},
  \end{align}
  combining \eqref{eq:B_x_k_x_star_angle_0} and \eqref{eq:V_k_2_norm_limit} in light of
  \eqref{eq:tilde_L_k_0} implies that $(\|x_k - x^*\|_{\mathcal{H}})_{k \in \mathbb{N}}$
  converges. Consequently, due to \cite[Lemma 2.47]{bauschke2017convex}, $(x_k)_{k \in \mathbb{N}}$
  converges weakly to a point in $\zer (A + B)$.
  \end{enumerate}

\end{IEEEproof}

\setcounter{equation}{0}
\renewcommand{\theequation}{\Alph{section}.\arabic{equation}}
\setcounter{theorem}{0}
\renewcommand{\thetheorem}{\Alph{section}.\arabic{theorem}}

\section{Proof of Proposition \ref{prop:step_size}}
\label{appendix:Proof of Proposition prop:step_size}

\begin{IEEEproof}
  When $\eta_k = \eta$ for all $k \in \mathbb{N} \cup \{-1\}$ for some $\eta \in \mathbb{R}_{++}$, \eqref{eq:assmp_A3-3} is equivalent that 
\begin{align}
  \Upsilon(\eta) \le \frac{1}{4} - \theta
  \Leftrightarrow \theta \le \frac{1}{4} - \Upsilon(\eta).
  \label{eq:Upsilon_inequality}
\end{align}
Hence, it is sufficient to show that \eqref{eq:Upsilon_inequality} holds for all $\eta \in [\eta_{-}, \eta_{+}]$.

First, the uniqueness and existence of $\eta_*$ is verified below.
Since $\Upsilon''(\eta) = 2 L_B^2  + 3\varpi \eta^{-3} > 0$ for any $\eta \in \mathbb{R}_{++}$, $\Upsilon$ is strictly convex over $\mathbb{R}_{++}$ \cite[Proposition 1.2.6]{bertsekas2003convex}.
Hence, by \cite[Proposition 2.1.2]{bertsekas2003convex}, there exists at most one global minimizer of $\Upsilon$ over $\mathbb{R}_{++}$.
Since $\lim_{\eta \downarrow 0} \Upsilon (\eta) = +\infty$, for all $M_1 > 0$, there exists $r_1 \in \mathbb{R}_{++}$ such that 
\begin{align}
  \Upsilon(\eta) > M_1,~ \forall \eta \in (0, r_1).
  \label{eq:eta_1_inequality}
\end{align}
Similarly, since $\lim_{\eta \to +\infty} \Upsilon (\eta) = +\infty$, for all $M_2 > 0$, there exists $r_2 \in \mathbb{R}_{++}$ such that 
\begin{align}
  \Upsilon(\eta) > M_2, ~\forall \eta \in (r_2, +\infty).
  \label{eq:eta_2_inequality}
\end{align}
Let $r_0 \in \mathbb{R}_{++}$, $r_{\min} \coloneqq \min\{r_1, r_2\}$, and $r_{\max} \coloneqq \max\{r_1, r_2\}$ for $r_1, r_2$ that satisfy \eqref{eq:eta_1_inequality} for $M_1 \coloneqq \Upsilon(r_0)$ and \eqref{eq:eta_2_inequality} for $M_2 \coloneqq \Upsilon(r_0)$, respectively.
Then, it holds that $\Upsilon(\eta) > \Upsilon(r_0)$ for all $\eta \in \mathbb{R}_{++} \setminus [r_{\min}, r_{\max}]$.
Since $[r_{\min}, r_{\max}]$ is compact and $\Upsilon$ is continuous on this set, there exists $\tilde{\eta}^* \in \mathbb{R}_{++}$ that minimizes $\Upsilon$ over $[r_{\min}, r_{\max}]$ \cite[Proposition 2.1.1]{bertsekas2003convex}.
Hence, by setting $\eta_* \coloneqq \tilde{\eta}^*$ if $\Upsilon(\tilde{\eta}^*) \le \Upsilon(r_0)$ and $\eta_* \coloneqq r_0$ otherwise, $\eta_*$ uniquely minimizes $\Upsilon$ over $\mathbb{R}_{++}$.

Since $\Upsilon'$ is monotonically increasing due to strict convexity and $\Upsilon' (\eta_*) = 0$, $\Upsilon$ is monotonically decreasing in $(0, \eta_*]$.
Hence, $\lim_{\eta \downarrow 0} \Upsilon(\eta) = +\infty$ and $\Upsilon(\eta_*) < 1 / 4 - \theta$ imply that $\Upsilon(\eta_{-}) = 1 / 4 - \theta$ for some $\eta_{-} \in (0, \eta_*]$, and \eqref{eq:Upsilon_inequality} holds for all $\eta \in [\eta_{-}, \eta_{*}]$.
Similarly, since $\Upsilon$ is monotonically increasing in $[\eta_*, +\infty)$, $\lim_{\eta \to +\infty} \Upsilon(\eta) = +\infty$ and $\Upsilon(\eta_*) < 1 / 4 - \theta$ imply that $\Upsilon(\eta_{+}) = 1 / 4 - \theta$ for some $\eta_{+} \in [\eta_*, +\infty)$, and \eqref{eq:Upsilon_inequality} holds for all $\eta \in [\eta_{*}, \eta_{+}]$.
\end{IEEEproof}


\end{document}